%% file: main.tex
\documentclass[11pt]{article}
\usepackage{amsmath,amssymb,amsfonts,amsthm,epsfig}
\usepackage[usenames,dvipsnames]{xcolor}
\usepackage{bm,xspace}
\usepackage{tcolorbox}
\usepackage{cancel}
\usepackage{fullpage}
\usepackage{liyang}
\usepackage{framed}
\usepackage{verbatim}
\usepackage{enumitem}
\usepackage{array}
\usepackage{multirow}
\usepackage{afterpage}
\usepackage{mathrsfs}
\usepackage{pifont} 
\usepackage{chngpage}
\usepackage[normalem]{ulem}
\usepackage{boxedminipage}
\usepackage{caption}
\usepackage{forest}

\usepackage{thm-restate}

\usepackage{algorithm}
\usepackage{algorithmicx}
\usepackage{algpseudocode}

\usepackage{dashbox}
\newcommand\emptydottedbox{%
    \setlength{\dashlength}{0.5pt}
    \setlength{\dashdash}{0.5pt}
    \setlength{\fboxrule}{0.5pt}
    \setlength{\fboxsep}{0pt}
    \dbox{\ensuremath{\phantom{x}}}%
    }

\newcommand{\anote}[1]{\footnote{{\bf \color{green}Ali}: {#1}}}
\newcommand{\gnote}[1]{\footnote{{\bf \color{violet}Guy}: {#1}}}

\def\colorful{0}

\ifnum\colorful=1
\newcommand{\violet}[1]{{\color{violet}{#1}}}

\newcommand{\gray}[1]{{\color{gray}{#1}}}

\fi
\ifnum\colorful=0
\newcommand{\violet}[1]{{{#1}}}

\newcommand{\gray}[1]{{{#1}}}

\fi

\renewcommand{\hat}{\wh} 

\newcommand{\mcC}{\mathcal C}
\newcommand{\mcD}{\mathcal D}
\newcommand{\mcE}{\mathcal E}
\newcommand{\mcM}{\mathcal M}
\newcommand{\mcA}{\mathcal A}
\newcommand{\mcB}{\mathcal B}

\newcommand{\mcX}{\mathcal X}
\newcommand{\mcY}{\mathcal Y}

\newcommand{\mcU}{\mathcal U}
\newcommand{\mcF}{\mathcal F}

\newcommand{\TV}{\dist_{\mathrm{TV}}}
\newcommand{\dadd}{\mathrm{cost}_{\mathrm{add}}}
\newcommand{\cost}{\mathrm{cost}}

\newcommand{\round}{\mathrm{round}}

\newcommand{\Stat}{\textsc{Stat}} 

\newcommand{\citet}{\cite}

\DeclarePairedDelimiter\ceil{\lceil}{\rceil}
\DeclarePairedDelimiter\floor{\lfloor}{\rfloor}

\newcommand{\paren}[1]{\left({#1}\right)}

\makeatletter
\newcommand{\doublehat}[1]{%
\begingroup%
  \let\macc@kerna\z@%
  \let\macc@kernb\z@%
  \let\macc@nucleus\@empty%
  \hat{\raisebox{.35ex}{\vphantom{\ensuremath{#1}}}\smash{\hat{#1}}}%
\endgroup%
}
\makeatother

\title{
On the power of adaptivity in statistical adversaries
\vspace{15pt}}
 \author{Guy Blanc \vspace{8pt} \\ \hspace{-5pt}{\sl Stanford}
 \and \hspace{0pt} Jane Lange \vspace{8pt} \\ \hspace{-4pt}  {\sl MIT}
 \and Ali Malik \vspace{8pt}\\ \hspace{-8pt} {\sl Stanford}
 \and Li-Yang Tan \vspace{8pt} \\ \hspace{-8pt} {\sl Stanford}
 }

\date{\vspace{15pt}\small{\today}}

\begin{document}

\maketitle

\begin{abstract} 
We study a fundamental question concerning adversarial noise models in statistical problems where the algorithm receives i.i.d.~draws from a distribution $\mathcal{D}$. The definitions of these adversaries specify the {\sl type} of allowable corruptions (noise model) as well as {\sl when} these corruptions can be made (adaptivity); the latter differentiates between oblivious adversaries that can only corrupt the distribution $\mathcal{D}$ and adaptive adversaries that can have their corruptions depend on the specific sample $S$ that is drawn from $\mathcal{D}$.

In this work, we investigate whether oblivious adversaries are effectively equivalent to adaptive adversaries, across all noise models studied in the literature.
Specifically, can the behavior of an algorithm~$\mathcal{A}$ in the presence of oblivious adversaries always be well-approximated by that of an algorithm $\mathcal{A}'$ in the presence of adaptive adversaries? Our first result shows that this is indeed the case for the broad class of  {\sl statistical query} algorithms, under all reasonable noise models. We then show that in the specific case of {\sl additive noise}, this equivalence holds for {\sl all} algorithms. Finally, we map out an approach towards proving this statement in its fullest generality, for all algorithms and under all reasonable noise models. 


\end{abstract} 

\thispagestyle{empty}
\newpage 
\setcounter{page}{1}

\input{intro}

\input{Preliminaries}

\input{AdversaryModel}

\input{SQ}

\input{Proofs}

\input{IfItCanBeDone}
\input{AdditiveSubsampling}

\input{Subsampling}

\section{Conclusion} 



\section*{Acknowledgements}

We thank Adam Klivans and Greg Valiant for helpful conversations.  We are grateful to Greg for allowing us to include~\Cref{thm:greg} which was proved jointly with him.   

Guy and Li-Yang are supported by NSF CAREER Award 1942123. Jane is
supported by NSF Award CCF-2006664. Ali is supported by a graduate fellowship award from Knight-Hennessy Scholars at Stanford University.

\bibliographystyle{alpha}
\bibliography{main}

\appendix
\input{TechnicalRemarks}
\input{appendix}

\end{document}

%% file: intro.tex
\section{Introduction}

The possibility of noise pervades most problems in statistical estimation and learning.  In this paper we will be concerned with {\sl adversarial} noise models, as opposed to the class of more benign random noise models.  Adversarial noise models are the subject of intensive study across statistics~\cite{Hub64,Ham71,Tuk75}, learning theory~\cite{Val85,Hau92,KL93,KSS94,BEK02}, and algorithms~\cite{DKKLMS19,LRV16,CSV17,DK19}.  The definition of each model specifies:
\begin{enumerate} 
\item The {\sl type} of corruptions allowed. For example, the adversary may be allowed to add arbitrary points (additive noise~\cite{Hub64,Val85}), or in the context of supervised learning, allowed to change the labels in the data (agnostic noise~\cite{Hau92,KSS94}).  
\item The {\sl adaptivity} of the adversary. 
\end{enumerate} 

The latter is the focus of our work. Consider any statistical problem where the algorithm is given i.i.d.~draws from a distribution~$\mathcal{D}$.  On one hand we have {\sl oblivious} adversaries: such an adversary corrupts $\mathcal{D}$ to a different distribution $\wh{\mathcal{D}}$, from which the algorithm then receives a sample.  On the other hand we have {\sl adaptive} adversaries: such an adversary first draws a sample $\bS$ from $\mathcal{D}$, and upon seeing the specific outcomes in~$\bS$, corrupts it to $\wh{\bS}$ which is then passed on to the algorithm.  One can further consider adversaries with intermediate adaptive power, but we think of this as a dichotomy for now. 
A coupling argument shows that adaptive adversaries are at least as powerful as oblivious ones~\cite{DKKLMS19,ZJS19}.  In this work we investigate whether they can be {\sl strictly} more powerful. 

\begin{question} 
\label{question} 
Fix the type of corruptions allowed.  Is it true that for any algorithm $\mathcal{A}$, there is an algorithm~$\mathcal{A}'$ whose behavior in the presence of adaptive adversaries well-approximates that of $\mathcal{A}$ in the presence of oblivious adversaries? 
\end{question} 


 
The distinction between oblivious and adaptive adversaries is frequently touched upon in works concerning statistical problems.  Sometimes this distinction is brought up in service of emphasizing that the algorithms given in these works are robust  against adaptive adversaries; other times it is brought up when the algorithms are shown to be robust against oblivious adversaries, and the viewpoint of an adaptive corruption process is provided as intuition for the noise model.  However, the relative power of oblivious and adaptive adversaries in the statistical setting has not been systematically considered in the literature. 

\subsection{\violet{Our contributions}} 

\subsubsection{\violet{A unified framework for characterizing data adversaries.}}

To reason generally about~\Cref{question}, we associate every type of allowable corruptions with a cost function $\rho$ between distributions.  An oblivious $\rho$-adversary therefore corrupts $\mathcal{D}$ to some $\wh{\mathcal{D}}$ that is $\eta$-close with respect to $\rho$, meaning that  $\rho(\mathcal{D}, \wh{\mathcal{D}}) \le \eta$. An adaptive $\rho$-adversary corrupts a sample $S$ drawn from $\mathcal{D}$ to some $\wh{S}$ such that the uniform distribution over $\wh{S}$ is $\eta$-close with respect to $\rho$ to that over $S$.\violet{For example, when $\rho(\mathcal{D}, \hat{\mathcal{D}})$ is the total variation distance between $\mcD$ and $\wh{\mcD}$, the resulting adaptive adversary represents {\sl nasty noise} as defined in~\citet{BEK02}, where the adversary is allowed to change an arbitrary $\eta$-fraction of the points in $S$. We discuss this framework in more detail in \Cref{section:models}}.

\subsubsection{Yes to~\Cref{question} for all SQ algorithms} 

Our first result is an affirmative answer to~\Cref{question} for the broad class of {\sl statistical query} (SQ) algorithms.  Our proof  will require a mild assumption on this cost function, the precise statement of which we defer to the body of the paper.  For now, we simply refer to cost functions satisfying this assumption as ``reasonable", and mention that it is easily satisfied by all standard noise models, and can be seen to be necessary for our result to hold.      

\begin{theorem}[\violet{SQ algorithms are robust to adaptive adversaries}]\footnote{\violet{See \Cref{thm:sq-formal} for the formal version of this theorem.}}
\label{thm:SQ} 
For all reasonable cost functions $\rho$ and SQ algorithms $\mathcal{A}$, the behavior of $\mcA' \coloneqq \mathcal{A}$ in the presence of adaptive $\rho$-adversaries well-approximates that of $\mathcal{A}$ in the presence of oblivious adversaries. 
\end{theorem}

A key ingredient in our proof of \Cref{thm:SQ} is a novel reduction, using duality, from a $k$-query SQ algorithm to a single ``representative" SQ.

In other words, the SQ framework neutralizes adaptive adversaries into oblivious ones, and in the context of~\Cref{question}, we can take $\mathcal{A}'$ to be $\mathcal{A}$ itself. Looking ahead to our other results, we remark that such a statement cannot be true for all algorithms: there are trivial examples of (non-SQ) algorithms $\mathcal{A}$ for which $\mathcal{A}'$ has to be a modified version of $\mathcal{A}$.\footnote{One such example is $\mathcal{D} = \mathrm{Bernoulli}(\frac1{2})$ and $\mathcal{A} = \Ind[\text{number of $1$'s in the sample is $0$ mod 100}]$.  As the sample size grows, an oblivious adversary can barely change the acceptance probability of $\mathcal{A}$, whereas an adaptive one can change it completely.} 

\Cref{thm:SQ}  adds to the already-deep connections between the SQ framework and noise tolerance.  The SQ framework was originally introduced in learning theory, where it continues to be influential in the design of learning algorithms that are resilient to {\sl random classification noise}~\cite{Kea98}. 
Most relevant to the topic of this paper, to our knowledge across all noise models, all existing statistical algorithms that have been shown to be robust to adversarial noise can be cast in the SQ framework.

\subsubsection{Yes to~\Cref{question} for additive  noise} 

Our second result is an affirmative answer to~\Cref{question} for one of the most natural  types of corruptions: 

\begin{theorem}\footnote{\violet{See \Cref{thm:add-formal} for the formal version}}
\label{thm:add} 
The answer to~\Cref{question} is ``yes" for additive noise. 
\end{theorem} 

In additive noise, an oblivious adversary corrupts $\mathcal{D}$ to $\wh{\mathcal{D}} = (1-\eta)\mathcal{D} + \eta\mathcal{E}$ for an arbitrary distribution $\mathcal{E}$ of their choosing. An adaptive adversary, on the other hand, gets to inspect the sample $S$ drawn from $\mathcal{D}$, and adds $\frac{\eta}{1-\eta}|S|$ many arbitrary points of their choosing to $S$.  The oblivious version of additive noise was introduced by Huber~\cite{Hub64} and has become known as Huber's contamination model; the adaptive version is commonly called ``data poisoning" in the security and machine learning literature.  

Additive noise also captures the well-studied {\sl malicious noise} (\Cref{def:mal-noise}) from learning theory~\cite{Val85} (see also~\cite{KL93}). 
\Cref{thm:add} therefore shows that Huber's contamination model, the malicious noise model, and the adaptive version of additive noise, are in fact all equivalent.  Our proof of~\Cref{thm:add} is constructive: we give an explicit description of how $\mathcal{A}'$ can be obtained from $\mathcal{A}$, and $\mathcal{A}'$ preserves the computational and sample efficiency of $\mathcal{A}$ up to polynomial factors.

\subsubsection{Yes to~\Cref{question} in its fullest generality?  An approach via subsampling} 

Our proof of~\Cref{thm:add} is actually an instantiation of a broader approach towards answering~\Cref{question} affirmatively in its fullest generality: showing that the answer is ``yes" for all (reasonable) types of allowable corruptions and all algorithms.  We introduce the following definition:

\begin{definition}[Neutralizing filter]
Let $\rho$ be a cost function and $\mathcal{A}$ be an algorithm for a statistical problem over a domain $\mcX$.  We say that a randomized function $\Phi : \mcX^*\to \mcX^*$ is a {\sl neutralizing filter for $\mathcal{A}$ with respect to $\rho$} if the following holds.  For all distributions $\mathcal{D}$, with high probability over the draw of  $\bS$ from $\mathcal{D}$, the behavior of \begin{center} $\mathcal{A}$ on $\Phi(\hat{\bS})$, where $\hat{\bS}$ is a corruption of $\bS$ by an adaptive $\rho$-adversary,\end{center} well-approximates the behavior of \begin{center}$\mathcal{A}$ on a sample from $\hat{\mathcal{D}}$, where $\hat{\mathcal{D}}$ is a corruption of $\mathcal{D}$ by an oblivious $\rho$-adversary.\end{center} 
\end{definition}

Perhaps the most natural filter in this context is the {\sl subsampling} filter.  For an $n$-sample algorithm $\mathcal{A}$, we request for a larger sample of size $m\ge n$, allow the adaptive adversary to corrupt it, and then run $\mathcal{A}$ on a size-$n$ subsample of the corrupted size-$m$ sample.  We call this the ``$m \to n$ subsampling filter" and denote it as $\Phi_{m\to n}$.  The hope here is for the randomness of the subsampling step to neutralize the adaptivity of the adversary.

The efficiency of the subsampling filter is measured by the overhead in sample complexity that it incurs, i.e.~how much larger $m$ is relative to $n$.  Subsampling from a  sample that is roughly the size of the domain of course makes adaptive adversaries equivalent to oblivious ones, but this renders a sample-efficient algorithm inefficient.  We are interested whether the subsampling filter can be effective while only incurring a mild overhead in sample complexity.  

We propose the following conjecture as a general approach towards answering~\Cref{question}:

\begin{conjecture}
\label{conj:bold} 
For all reasonable cost functions $\rho$ and $n$-sample algorithms $\mcA$, the subsampling filter $\Phi_{m\to n}$ is a neutralizing filter for $\mcA$ with respect to $\rho$ with $m = \poly(n,\log(|\mcX|))$. 
\end{conjecture}

We obtain~\Cref{thm:add} by proving~\Cref{conj:bold} in the case of $\rho$ being additive noise. 
While we have not been able to prove~\Cref{conj:bold} for all $\rho$'s and all algorithms, we can show the following:

\begin{theorem}[\violet{If it is possible, subsampling neutralizes adaptivity}]\footnote{\violet{See \Cref{thm:if-can-be-done-formal} for the formal version}}
\label{thm:if-can-be-done} 
Let $\rho$ be a cost function and $\mathcal{A}$ be an $n$-sample algorithm.  Suppose there is an $m$-sample algorithm $\mathcal{A}'$ whose behavior in the  presence of adaptive $\rho$-adversaries well-approximates that of $\mathcal{A}$ in the presence of oblivious $\rho$-adversaries.  Then $\Phi_{M\to n}$ is a neutralizing filter for $\mathcal{A}$ with respect to $\rho$ with $M = O(m^2)$. 
\end{theorem}

We note that in the context of~\Cref{thm:if-can-be-done}, we do not require $\mathcal{A}'$ to be computationally efficient: as long as $\mathcal{A}'$ is sample efficient, then our resulting algorithm,  the subsampling filter applied to $\mathcal{A}$, inherits the computational efficiency of $\mathcal{A}$. 

Finally, we show that the bound on $m$ in~\Cref{conj:bold} cannot be further strengthened to be independent of $|\mathcal{X}|$, the size of the domain: 


\begin{theorem}[\violet{Subsampling lower bound}]\footnote{\violet{See \Cref{thm:subsample} for the formal version}}
\label{thm:greg} 
Let $\rho$ be the cost function for additive noise and $\eta$ be the corruption budget.  There is an $n$-sample algorithm $\mathcal{A}$ such that for  $m \le O_\eta(n \log(|\mcX|)/ \log^2 n)$,  $\Phi_{m\to n}$ is {\sl not} a neutralizing filter for $\mathcal{A}$ with respect to $\rho$.


\end{theorem}

While~\Cref{thm:greg} shows that some dependence on $|\mcX|$ is necessary, we remark that the $\log(|\mcX|)$ dependence in~\Cref{conj:bold} is fairly mild---this is the description length of a sample point $x\in \mcX$.  \Cref{thm:greg} also shows that the quantiative bounds that we achieve for~\Cref{thm:add} has an optimal dependence on $|\mathcal{X}|$.

\input{related}

\subsection{Discussion and future work}

\paragraph{Implications of our results} \Cref{thm:SQ} says that for all reasonable cost functions, all SQ algorithms (existing and future ones) that are resilient to oblivious adversaries are ``automatically" also resilient to adaptive adversaries.  Likewise, lower bounds against adaptive adversaries immediately yield lower bounds against oblivious ones.  The same remark further applies for all algorithms in the case of additive noise and malicious noise, by~\Cref{thm:add}.  

As a concrete example, we recall that the agnostic learning framework was originally defined with respect to oblivious adversaries~\cite{Hau92,KSS94}.  As in the PAC model there a concept class $\mathcal{C}$, but the target function $f$ is no longer assumed to lie within $\mathcal{C}$---hence the name of the model.  The learning algorithm is expected to achieve error close to $\opt$, the distance from $f$ to $\mathcal{C}$.  However, many papers on agnostic learning  provide the viewpoint of an adaptive corruption process as intuition for the model: the data {\sl is} assumed to be a labeled according to a function $f\in \mathcal{C}$, but an adversary corrupts an $\opt$ fraction of the labels given to the learning algorithm.  This adaptive version was subsequently defined as a separate model called {\sl nasty classification noise}~\cite{BEK02} (as a special case of the nasty sample noise model introduced in that paper).  \Cref{thm:SQ} therefore shows that the agnostic learning model and the nasty classification noise model are in fact equivalent when it comes to SQ algorithms.

\paragraph{When can quantifiers be swapped?} The distinction between oblivious and adaptive adversaries can be viewed as a difference in the order of ``for all" and ``with high probability" quantifiers in the performance guarantees of statistical algorithms.   An algorithm succeeds in the presence of oblivious adversaries if {\sl for all} distributions $\wh{\mathcal{D}}$ that are close to $\mathcal{D}$, the algorithm succeeds {\sl with high probability} over a sample $\bS$ drawn from $\wh{\mathcal{D}}$.  On the other hand, an algorithm  succeeds in the presence of adaptive adversaries if {\sl with high probability} over a sample $\bS$ drawn from $\mathcal{D}$, the algorithm succeeds {\sl for all} corruptions $\wh{\bS}$ that are close to $\bS$.  
Our work formalizes the question of when these quantifiers can be swapped, and our results provide several answers.   

\paragraph{Future work}

In this work we initiate the systematic study of the power of adaptivity in statistical adversaries.  A concrete direction for future work is to answer~\Cref{question} for other broad classes of algorithms and natural noise models, either via the subsampling filter (\Cref{conj:bold}) or otherwise.  Here we highlight the specific case of subtractive noise: having resolved the case of additive noise in this work, doing so for subtractive noise as well would be a significant step towards resolving~\Cref{question} for all three generic noise models described in~\Cref{section:models}.

%% file: related.tex
\subsection{Other related work}

\paragraph{A separation result} Recently, Deng et al.\ gave a separation between the adaptive adversary, which they call ``data-aware", and the oblivious adversary \cite{DGJMMG21}. Specifically they showed that there are settings in which a natural Lasso-based algorithm for feature selection will often fail to select the correct features in the presence of an adaptive additive adversary, but would succeed in the presence of an oblivious additive adversary. Our \Cref{thm:add} implies that this Lasso-based algorithm would succeed if it used the subsampling filter to preprocess its sample.

\paragraph{The online and dynamic setting}
While the focus of our work is on the statistical setting, the distinction between adaptive and oblivious adversaries has also been the subject of recent study in the {\sl online} \cite{HRS21,ABDMNY21} and {\sl dynamic} \cite{BKMKSS21} setting, albeit with a notably different notion of adaptivity. In these settings, the adaptive adversaries can change the input distribution in response to the previous behavior of the algorithm, while oblivious adversaries must choose a fixed input distribution before the algorithms run.

\paragraph{Adaptive data analysis} We emphasize the distinction between the focus of our work and the recent fruitful line of work on adaptive data analysis~(\citet{DFHPRR15,HU14,SU15,BNSSSU21}). The focus of our work is on the adaptivity of the {\sl adversary}, whereas the focus of this line of work is on the adaptivity of the {\sl SQ algorithm}. Throughout this work, we reserve the use of ``adaptive" to refer to the adversary, and all SQ algorithms will inherently be adaptive. 

%% file: Preliminaries.tex
\section{Preliminaries}
We use {\bf boldface} (e.g.~$\bx \sim \mcD$) to denote random variables.  Throughout this paper $\mcX$ denotes an arbitrary finite domain\footnote{See \Cref{sec:technical remarks} for remarks regarding infinite domains.}, and we write $S \in \mcX^*$ to represent a multiset of elements in $\mcX$, 
meaning $S \in \mcX^{0} \cup \mcX^{1} \cup \mcX^{2} \ldots$. \violet{We use the notation $a = b \pm \eps$ to indicate that $|a - b| < \eps$. For any $m \in \N$, the notation $[m]$ indicates the set $\{1,2,\ldots, m\}$}.

\paragraph{Distributions.} For any $S \in \mcX^*$, we use $\mathcal{U}(S)$ to refer to the uniform distribution over $S$. For simplicity, we  enforce that all distributions only have rational probabilities, meaning $\Prx_{\bx \sim \mcD}[\bx = x]$ is rational for any distribution $\mcD$ and element $x \in \mcX$.\footnote{Alternatively, one could enforce that the cost function smoothly interpolate to irrational probabilities, which is the case for all standard noise models.}  For any distributions $\mcD_1, \mcD_2$ and parameter $\theta \in [0,1]$, we use $\theta \mcD_1 + (1-\theta)\mcD_2$ to refer to the mixture distribution which samples from $\mcD_1$ with probability $\theta$ and from $\mcD_2$ with probability $(1 - \theta)$

\begin{definition}[Total variation distance]
    \label{def:tv-distance}
    Let $\mcD$ and $\mcD'$ be any two distributions over the same domain, $\mcX$. It is well known that the following are equivalent definitions for the \emph{total variation distance} between $\mcD$ and $\mcD'$, denoted $\TV(\mcD, \mcD')$:
    \begin{enumerate}
        \item It is characterized by the best test distinguishing the two distributions:
        \begin{equation*}
            \TV(\mcD, \mcD') \coloneqq \sup_{T: \mcX \to [0,1]} \left\{ \Ex_{\bx \sim \mcD}[T(\bx)] -\Ex_{\bx \sim \mcD'}[T(\bx')]   \right\}.
        \end{equation*}
        \item It is characterized by the coupling which makes the two random variables different with the smallest probability.
        \begin{equation*}
            \TV(\mcD, \mcD') \coloneqq \inf_{(\bx, \bx') \text{ a coupling of } \mcD, \mcD'} \left\{ \Prx[\bx \neq \bx']   \right\}. 
        \end{equation*}
    \end{enumerate}
\end{definition}

%% file: AdversaryModel.tex
\section{Adversarial noise models}
\label{section:models}

To reason generally about various noise adversarial models, we represent the types of allowable corruptions by a budget $\eta$ and cost function $\rho$, which maps each ordered pair of distributions to some non-negative cost (or infinity). The cost need not be symmetrical. 


\begin{definition}[$(\rho,\eta)$-oblivious adversary]
    \label{def:oblivious}
    Given some cost function $\rho$ and budget $\eta$, an algorithm operating in the oblivious adversary model will receive the following input: If the true data distribution is $\mcD$, then the algorithm will receive iid samples from an adversarial chosen $\wh{\mcD}$ satisfying $\rho(\mcD, \wh{\mcD}) \leq \eta$.
\end{definition}

\begin{definition}[$(\rho,\eta)$-adaptive adversary]
    \label{def:adaptive}
    Given some cost function $\rho$ and budget $\eta$, an $n$-sample algorithm operating in adaptive adversary model will receive the following input: If the true data distribution is $\mcD$, first a clean sample $\bS \sim \mcD^n$ is generated, and then the algorithm will get an adversarially chosen $\wh{\bS}$ satisfying $\rho(\mcU(\bS), \mcU(\wh{\bS})) \leq \eta$.
\end{definition}

Throughout this paper, we use $\wh{\emptydottedbox}$ to denote the corrupted version of a set or distribution.

We remark that the adaptive adversary as defined in~\Cref{def:adaptive} is slightly stronger than the definitions usually considered, in the sense that  there is usually a bound on the size of $\wh{S}$ the adversary is allowed to produce. For example, in the nasty noise model (\Cref{def:strong-cont}), the adversary is only allowed to change points in the sample, and so $|\wh{S}| = |S|$. All of our results apply regardless of the size of $\wh{S}$.






\subsection{Standard noise models from the literature}

In this subsection, we present three generic adversary models and show how they are special cases of our framework with an appropriate choice of cost function. Other standard models, and how they fit within our framework, are given in \Cref{sec:other-noise-models}.



\begin{definition}[Additive noise]
    \label{def:additive-cont}
    Given a size-$n$ sample $S \in \mcX^n$ and a corruption budget $\eta$, the adaptive additive noise adversary is allowed to add $\floor{n \cdot \eta/(1-\eta)}$ points to $S$ arbitrarily. 
\end{definition}
The additive noise model is captured by the cost function: 
\begin{align*}
     \mathrm{cost}_{\mathrm{add}}(\mcD, \wh{\mcD}) \coloneqq \inf_{\eta \in \R_{\geq 0}} \left \{ \wh{\mcD} = (1 - \eta)\mcD + \eta \mcE \right \} \quad \text{ for some distribution }\mcE.
\end{align*}
The oblivious version of additive noise is the well-known Huber contamination model~\cite{Hub64}. 

\begin{definition}[Subtractive noise]
    Given a size-$n$ sample $S \in \mcX^n$ and a corruption budget $\eta$, the adaptive subtractive noise adversary is allowed to remove $\floor{\eta n}$ points from $S$ arbitrarily. 
\end{definition}

The subtractive noise adversary is captured by the cost function: 
\begin{align*}
         \mathrm{cost}_{\mathrm{sub}}(\mcD, \wh{\mcD}) \coloneqq \dadd(\wh{\mcD}, \mcD).
    \end{align*}

\begin{definition}[Nasty noise]
    \label{def:strong-cont}
    Given a size-$n$ sample $S \in \mcX^n$ and a corruption budget $\eta$, the adaptive nasty noise adversary is allowed to change up to $\floor{\eta n}$ of the points in $S$ arbitrarily.
\end{definition}
This noise model is also known as {\sl strong contamination}, and as {\sl nasty sample noise} (or simply nasty noise) in the context of supervised learning~\cite{BEK02}.  It is captured by the cost function $\rho = \TV$.

%% file: SQ.tex
\section{Proof of \Cref{thm:SQ}: The SQ framework neutralizes adaptive adversaries}\label{sec:SQ}

In this section we formally state and prove~\Cref{thm:SQ}---namely that the behaviour of a Statistical Query (SQ) algorithm in the presence of adaptive adversaries is equivalent to its behaviour in the presence of oblivious adversaries. 

\paragraph{Basics of the SQ framework.} Let $\mcD$ be a distribution over some domain $\mcX$. A statistical query is a pair $(\phi,\tau)$ where $\phi : \mcX \to [-1,1]$ is the query and $\tau > 0$ is a tolerance parameter. These queries can be answered by a statistical query oracle $\Stat_\mcD$ which, given an SQ $(\phi, \tau)$, returns a value $v$ equal to $\Ex_{\bx \sim \mcD}[\phi(\bx)]$ up to an additive error of $\tau$ i.e. $v \in \Ex_{\bx \sim \mcD}[\phi(\bx)] \pm \tau$.
    
Throughout this section, we will use some convenient shorthand.  For any distribution $\mcD$ and multiset $S \in \mcX^*$, we write
\begin{equation*}
    \phi(\mcD) \coloneqq \Ex_{\bx\sim\mcD}[\phi(\bx)] \quad \quad \text{and} \quad\quad \phi(S) \coloneqq \phi(\mcU(S)) = \frac1{|S|} \sum_{x \in S} \phi(x).
\end{equation*}
A $k$-query statistical query algorithm, $\mcA$, is an algorithm that makes a sequence of statistical queries to the oracle $\Stat_\mcD$ one by one, using the result of the previous queries to decide which statistical query to make next.

\begin{definition}[$k$-query SQ Algorithm]
A $k$-query SQ algorithm $\mcA$ is a sequence of $k$ SQs
\[ (\phi^{(1)},\tau^{(1)}), (\phi^{(2)}_{v_1},\tau^{(2)}), (\phi^{(3)}_{ v_1, v_2},\tau^{(3)}),\ldots, (\phi^{(k)}_{v_1, \ldots, v_{k-1}},\tau^{(k)})\] 
to $\Stat_\mathcal{D}$, where $\mathcal{A}$'s choice of the $(i+1)$-st SQ can depend on $v_1,\ldots,v_{i}$ which are the answers of $\Stat_\mathcal{D}$ to the previous $i$ SQs.  For notational simplicity, we make the standard assumption that all the $\tau$'s are the same.

\end{definition}

\paragraph{Using mechanisms to implement $\Stat_{\mathcal{D}}$.} The SQ framework is a stylized model that cleanly facilitates theoretical analyses; It allows the algorithm designer to abstract away an algorithm's interaction with a random sample and instead \emph{assume} $\phi(\mcD)$ can be accessed up to $\pm \tau$ accuracy.

For an SQ algorithm to be useful, the $\Stat_{\mcD}$ oracle must be implemented. This is done by a \emph{mechanism} which uses a random sample $\bS \sim \mcD^n$ to simulate $\Stat_{\mcD}$ with high probability. The interaction between a mechanism and SQ algorithm is depicted in \Cref{fig:mechanism}.

\begin{figure}[H]
  \captionsetup{width=.9\linewidth}
\begin{tcolorbox}[colback = white,arc=1mm, boxrule=0.25mm]

\begin{enumerate}[nolistsep,itemsep=2pt]
    \item Fix some $k$-query SQ algorithm $\mcA$ that is unknown to the mechanism $\mcM$ and a distribution $\mcD$ that is unknown to both $\mcA$ and $\mcM$.
    \item Draw a sample $\bS \sim \mcD^n$ that is revealed to $\mcM$ but not $\mcA$.
    \item For $i = 1, \ldots, k$,
    \begin{enumerate}[nolistsep,itemsep=2pt]
        \item $\mcA$ chooses a query $\phi^{(i)}$ (as a function of responses to previous queries).
        \item $\mcM$ chooses a response $v_i$ for the query which is revealed to $\mcA$.
    \end{enumerate}
\end{enumerate}

\end{tcolorbox}
\caption{The interaction between a mechanism $\mcM$ and SQ algorithm $\mcA$.}
\label{fig:mechanism}
\end{figure}

\begin{definition}[$(\tau, \delta)$-accurate mechanisms]
    \label{def:accurate-mech}
    A mechanism $\mcM$ is $(\tau, \delta)$-accurate for $k$-query SQ algorithms, if for any distribution $\mcD$ and SQ algorithm $\mcA$, with probability at least $(1-\delta)$ over the randomness of $\bS$ and $\mcM$,
    \begin{equation*}
        v_i = \phi^{(i)}(\mcD) \pm \tau \quad \text{for all }i = 1,\ldots, k
    \end{equation*}
    where $v_i$ and $\phi^{(i)}$ are defined as in \Cref{fig:mechanism}.
\end{definition}

In this work, we focus on the $\tau$-rounding mechanism.
\begin{definition}[$\tau$-rounding mechanism]
\label{def:tau-rounding-mech}
    Given a sample $S \in \mcX^n$ and query $\phi$, the $\tau$-rounding mechanism, denoted $\mcM_{\tau}$, returns the answer $v = \round(\phi(S), \tau)$ where $\round(x, \tau)$ refers to $x$ rounded to the nearest integer multiple of $\tau$.
\end{definition}

\begin{fact}[The $\tau$-rounding mechanism is accurate]
    \label{fact:tau-rounding-accurate}
    For any $k \in \N$ and $\delta, \tau > 0$, the $\tau$-rounding mechanism with a sample size of 
    \begin{equation*}
        n = O\left(\frac{k\log(1/\tau) +  \log(1/\delta)}{\tau^2}\right)
    \end{equation*}
    is $(\tau,\delta)$ accurate for $k$-query SQ algorithms.
\end{fact}
\begin{proof}
    For each query, the $\mcM_\tau$ can return one of only $O(1/\tau)$ possible values (after rounding). The $i^{\mathrm{th}}$ query is chosen as a function of $v_1, \ldots, v_{i-1}$, so there are at most $O(1/\tau)^{i-1}$ possible choices for the $i^{\mathrm{th}}$ query, and only $O(1/\tau)^k$ total unique queries $\mcA$ could choose. Using a Chernoff bound and union bound over all possible queries, the probability $\mcA$ asks a query, $\phi$, where $|\phi(\bS)- \phi(\mcD)| \geq \tau/2$ is at most 
    \begin{equation*}
        \exp\left(-\Omega(\tau^2 n) + O(k \log(1/\tau))\right)
    \end{equation*}
    For the $n$ given in \Cref{fact:tau-rounding-accurate}, that probability is at most $\delta$. The desired result follows from triangle inequality and $|\round(\phi(\bS), \tau) - \phi(\bS)| \leq \tau/2$.
\end{proof}

The existence of accurate mechanisms (as in \Cref{fact:tau-rounding-accurate}) is the key to the SQ framework: SQ algorithms can assume that they have access to a $\Stat_{\mcD}$ oracle, because for modest sample sizes and tiny failure probabilities, mechanisms are a $\Stat_{\mcD}$ oracle.

\subsection{The SQ framework in the presence of adversarial noise.}

The SQ framework naturally extends to oblivious adversaries.

\begin{definition}[$k$-query SQ Algorithm with an oblivious adversary]
Fix a cost function $\rho$ and budget $\eta$. An $k$-query SQ algorithm, $\mcA$, in the presence of a $(\rho, \eta)$-oblivious adversary is a sequence of $k$ SQs
\[ (\phi^{(1)},\tau), (\phi^{(2)}_{v_1},\tau), (\phi^{(3)}_{ v_1, v_2},\tau),\ldots, (\phi^{(k)}_{v_1, \ldots, v_{k-1}},\tau)\] 
each of which are answered according to $\Stat_{\wh{\mcD}}$ for some $\wh{\mcD}$ satisfying $\rho(\mcD, \wh{\mcD}) \leq \eta$. This $\wh{\mcD}$ is the same for all queries, but adversarially chosen. $\mathcal{A}$'s choice of the $(i+1)$-st SQ can depend on $v_1,\ldots,v_{i}$ which are the answers of $\Stat_{\wh{\mathcal{D}}}$ to the previous $i$ SQs. 

\end{definition}

Once again, to run an SQ algorithm, the $\Stat_{\wh{\mcD}}$ oracle is implemented by a mechanism. Given a sample $\bS \sim \wh{\mcD}$, any mechanism satisfying \Cref{def:accurate-mech} will be able to simulate the $\Stat_{\wh{\mcD}}$ oracle.

In the presence of $(\rho, \eta)$-\emph{adaptive} adversaries, first a clean sample $\bS \sim \mcD^n$ is drawn, and then the adversary chooses an $\eta$-corruption $\wh{\bS}$ that is passed to the mechanism, as shown in \Cref{fig:mechanism-adaptive}.
\begin{figure}[H]
  \captionsetup{width=.9\linewidth}
\begin{tcolorbox}[colback = white,arc=1mm, boxrule=0.25mm]

\begin{enumerate}[nolistsep,itemsep=2pt]
    \item Fix some $k$-query SQ algorithm $\mcA$ that is unknown to the mechanism $\mcM$ and a distribution $\mcD$ that is unknown to both $\mcA$ and $\mcM$.
    \item Draw a sample $\bS \sim \mcD^n$ that is revealed to neither $\mcA$ or $\mcM$.
    \item An adversary chooses an $\wh{\bS}$ that is $\eta$-close to $\bS$ which is revealed to $\mcM$ but not $\mcA$.
    \item For $i = 1, \ldots, k$,
    \begin{enumerate}[nolistsep,itemsep=2pt]
        \item $\mcA$ chooses a query $\phi^{(i)}$ (as a function of responses to previous queries).
        \item $\mcM$ chooses a response $v_i$ for the query which is revealed to $\mcA$.
    \end{enumerate}
\end{enumerate}

\end{tcolorbox}
\caption{The interaction between a mechanism $\mcM$ and SQ algorithm $\mcA$ in the presence of an adaptive adversary.}
\label{fig:mechanism-adaptive}
\end{figure}

Our goal is to show that there are mechanisms that can simulate $\Stat_{\wh{\mcD}}$ for some $\wh{\mcD}$ $\eta$-close to $\mcD$ given just the corrupted sample $\wh{\bS}$.

\begin{definition}[$(\tau, \delta)$-accurate in the presence of adaptive noise]
    \label{def:accurate-mech-noise}
    Fix a cost function $\rho$, budget $\eta$. A mechanism $\mcM$ is said to be $(\tau, \delta)$-accurate for $k$-query SQ algorithms in the presence of adaptive noise, if for any distribution $\mcD$ and SQ algorithm $\mcA$, the following holds. With probability at least $1 - \delta$ over the randomness of $\bS \sim \mcD^n$, 
    \begin{equation*}
        v_i = \phi^{(i)}(\wh{\mcD}) \pm \tau \quad \text{for all }i = 1,\ldots, k
    \end{equation*}
      for some $\wh{\mcD}$ $\eta$-close to $\mcD$, where $v_i$ and $\phi^{(i)}$ are defined as in \Cref{fig:mechanism-adaptive}. In particular, this holds regardless of how the adversary chooses $\wh{\bS}$.
\end{definition}

If $\mcA$ succeeds given $\Stat_{\wh{\mcD}}$ for every distribution $\wh{\mcD}$ that is $\eta$-close to $\mcD$ (i.e. $\mcA$ is resilient to oblivious adversaries), then, with high probability, $\mcA$ also succeeds in the presence of an adaptive adversary when using a mechanism satisfying \Cref{def:accurate-mech-noise}. We will show that the rounding mechanism meets \Cref{def:accurate-mech-noise} whenever $\rho$ is ``reasonable" in the sense of the following definition; this property is easily satisfied by all standard cost functions.

\begin{definition}[Closed under mixtures]
    \label{def:closed-under-mixtures}
    We say that $\rho$ is \emph{closed under mixtures} if for any distributions $\mcD_1, \mcD_2, \wh{\mcD}_1, \wh{\mcD}_2$ and $\theta \in (0,1)$,
    \begin{align*}
        \rho(\theta \mcD_1 + (1 - \theta)\mcD_2,\theta \wh{\mcD}_1 + (1 - \theta)\wh{\mcD}_2) \leq \max(\rho(\mcD_1, \wh{\mcD}_1), \rho(\mcD_2, \wh{\mcD}_2)).
    \end{align*}
\end{definition}
 
Requiring that $\rho$ is closed under mixtures enforces that the adaptive and oblivious adversaries ``match up" in the sense of making the same types of changes. This is formalized in the following fact, for which we provide a short proof in \Cref{appendix:sample-to-dist}.

\begin{fact}
    \label{fact:sample-to-dist}
    Let $\rho$ be closed under mixtures, $\mcD$ be a distribution over $\mcX$,   $\eta$ be a corruption budget, and $n\in \N$.  Suppose that for all $S \in \mathcal{X}^n$, there is a corresponding $\wh{S}$ satisfying $\rho(\mcU(S), \mcU(\wh{S})) \leq \eta$.  Let $\wh{\mcD}$ be the distribution where $\bx \sim \wh{\mcD}$ is generated by: 1) Drawing $\bS \sim \mcD^n$ and 2), drawing $\bx \sim \mcU(\wh{\bS})$. Then, $\rho(\mcD, \wh{\mcD}) \leq \eta$.
\end{fact}

Our quantitative bounds will depend on a parameter that is related to the types of corruption the adversaries can make. Suppose that the adaptive adversary is required to keep the size of the corrupted sample the same as the clean sample ($|\hat{S}| = |S|$). In this case, we require that if two samples $S_1$ and $S_2$ differ in only one point, then for any $\hat{S}_1$ that is $\eta$-close to $S_1$, there is some $\hat{S}_2$ that is $\eta$-close to $S_2$ where $\hat{S}_1$ and $\hat{S}_2$ differ in only a small number of points. The following definition generalizes that notion to the case where the adversary can also change the number of points in the sample.

\begin{definition}[$\ell$-local] 
    \label{def:local}
    For any $\ell > 0$, a cost function $\rho$ with budget $\eta$ is $\ell$-local if for any distributions $\mcD_1, \mcD_2$ and $\eta$-corruption $\wh{\mcD}_1$ of $\mcD_1$, there is some $\eta$-corruption $\wh{\mcD}_2$ of $\mcD_2$ satisfying $ \TV(\wh{\mcD}_1, \wh{\mcD}_2) \leq \ell \cdot \TV(\mcD_1, \mcD_2)$
\end{definition}

 All of the adversary models in \Cref{section:models} are $1$-local with the exception of the $\eta$-subtractive noise, which is $\frac{1}{1 - \eta}$-local. We encourage the reader to think of $\ell$ as a constant. 
 
 We are now ready to state the formal version of~\Cref{thm:SQ}, which generalizes~\Cref{fact:tau-rounding-accurate} to the setting of adversarial noise.

\begin{theorem}[Formal version of \Cref{thm:SQ}] 
\label{thm:sq-formal} 
For any $\ell$-local cost function, adversary budget $\eta$, $\delta, \tau > 0$ , and $k \in \N$, the $\tau$-rounding mechanism with a sample size of
\begin{equation*}
        n = O\left(\frac{\ell^2(k\log(1/\tau) +  \log(1/\delta))}{\tau^2}\right)
\end{equation*}
is $(\tau,\delta)$ accurate for $k$-query SQ algorithm in the presence of adaptive noise.
\end{theorem} 

\paragraph{Proof sketch.}
Here, we prove \Cref{thm:sq-formal} contingent on \Cref{lem:one-query,lem:many-sqs-to-one}, which we prove in \Cref{sec:SQ-proofs}.

First, we will prove the special case where $\mcA$ makes only a single query.\footnote{To prove \Cref{thm:sq-formal} in the case of a single statistical query, we could apply \Cref{lem:one-query} twice: Once to bound how large the adversary can make $\Psi(\hat{\bS})$ and once to bound how small it can make $\Psi(\hat{\bS})$. Instead, we will directly apply \Cref{lem:one-query} to prove the multi-query case of \Cref{thm:sq-formal}.}

\begin{restatable}{lemma}{onequery}\textbf{\emph{(\Cref{thm:sq-formal} in the case of a single SQ)}}
    \label{lem:one-query}
Let $\Psi:\mcX \to [-1,1]$ be a statistical query, $T \in [-1,1]$, and suppose: 
    \[
        \Psi(\wh{\mcD}) \leq T \quad\quad \text{for all $\wh{\mcD}$ that are $\eta$-close to $\mcD$}.
   \]
    Then, for any $\tau > 0$ and sample size $n \in \N$, the probability over $\bS \sim \mcD^n$ that there is some $\wh{\bS}$ that is $\eta$-close to $\bS$ satisfying
    \[
        \Psi(\wh{\bS}) \geq T + \frac{\tau}{2}
    \]
    is at most $\exp\left(-\frac{\tau^2n}{8\ell^2} \right)$. 
\end{restatable}
In order to prove \Cref{thm:sq-formal}, we want to bound the probability that for a random $\bS \sim \mcD^n$, there is some corruption $\wh{\bS}$ for which $\mcM_{\tau}$ is not $\tau$-accurate. In detail, that means, for
\begin{equation}
\label{eq:set-v}
\begin{split}
    v_1 &\coloneqq \round(\phi^{(1)}(\wh{\bS}),\tau) \\
    v_{i+1} &\coloneqq \round(\phi^{(i+1)}_{v_1, \ldots, v_i}(\wh{\bS}),\tau) \quad \text{for }i\in \{0,1,\ldots,k-1\},
\end{split}
\end{equation}
and, for every $\hat{\mcD}$ that is $\eta$-close to $\mcD$, there is some $i$ for which $|v_i - \phi^{(i)}_{v_1, \ldots, v_{i-1}}(\wh{\mcD})| > \tau$. Consider a single possible choice for $v_1, \ldots, v_k$ (which also fixes the $k$-statistical queries, $\phi^{(1)}, \ldots, \phi^{(k)}$). We use the separating hyperplane theorem to reduce to the case of a single statistical query.

\begin{restatable}{lemma}{duality}\textbf{\emph{($k$-query SQ algorithm to a single SQ)}}
    \label{lem:many-sqs-to-one}
    Fix any $k$ statistical queries $\phi^{(1)}, \ldots, \phi^{(k)}:\mcX \to [-1,1]$ and $k$ values $v_1, \ldots, v_k$. Suppose that there is no $\wh{\mcD}$ that is $\eta$-close to $\mcD$ satisfying
    \begin{align*}
        \phi^{(i)}(\wh{\mcD}) \in v_i \pm \tau \quad \quad \text{for every }i \in [k].
    \end{align*}
    Then there exists a single statistical query $\Psi: \mcX \to [-1,1]$ and threshold $T$ with the following properties.
    \begin{enumerate}
        \item \label{criteria:oblivious-small}  $\Psi(\wh{\mcD}) \leq T$ for every $\wh{\mathcal{D}}$ that is $\eta$-close to $\mathcal{D}$. 
        \item \label{criteria:adaptive-large} For any sample $\wh{S}$ satisfying $\phi^{(i)}(\wh{S}) \in v_i \pm \frac{\tau}{2}$ for each $i \in [k]$, it is also true that $\Psi(\wh{S}) \geq T + \frac{\tau}{2}$.
    \end{enumerate}
\end{restatable}
Applying~\Cref{lem:one-query,lem:many-sqs-to-one}, for any fixed choice of $v_1, \ldots, v_k$, the probability there is some $\wh{\bS}$ $\eta$-close to a sample $\bS \sim \mcD^n$ satisfying \Cref{eq:set-v}, and for every $\hat{\mcD}$ that is $\eta$-close to $\mcD$, there is some $i$ for which $|v_i - \phi^{(i)}_{v_1, \ldots, v_{i-1}}(\wh{\mcD})| > \tau$ is only $\exp\left(-\frac{\tau^2n}{8\ell^2} \right)$. Since each $v_i \in [-1,1]$ is an integer multiple of $\tau$, there are at most $(\frac{2}{\tau}+1)^k$ many choices for $(v_1,\ldots,v_k)$.  A union bound over all these choices completes the proof of~\Cref{thm:sq-formal}.

%% file: Proofs.tex
\subsection{Proofs of \Cref{lem:one-query,lem:many-sqs-to-one}}
\label{sec:SQ-proofs}

We will use a few standard technical tools:

    \begin{fact}[Separating hyperplane theorem]
        Let $A, B \in \R^k$ be disjoint, nonempty, and convex. There exists a nonzero vector $w \in \R^k$ and $T \in \R$ such that $a \cdot w \leq T$ and $b \cdot w \geq T$ for all $a \in A$ and $b \in B$.
    \end{fact}

    \begin{fact}[McDiarmid's inequality]
        \label{fact:mcdiarmid}
        Suppose that $f: \mcX^n \to \R$ satisfies the $c$-bounded difference property: for any $(x_1, \ldots, x_n), (x_1', \ldots, x_n') \in \mcX^n$ that differ on only on a single coordinate, $f$ satisfies
        \[ 
            \left|f(x_1, \ldots, x_n) - f(x_1', \ldots x_n')\right| \leq c.
        \]
        Then, for any $\tau > 0$ and any distribution $\mcD$ over $\mcX$,
        \begin{align*}
            \Prx_{\bS \sim \mcD^n}[f(\bS) - \mu \geq \tau] \leq \exp\left(-\frac{2 \tau^2}{c^2n} \right) \quad\quad\quad\quad \text{where} \quad \mu \coloneqq \Ex_{\bS \sim \mcD^n}[f(\bS)].
        \end{align*}
    \end{fact}




We prove the following two Lemmas, restated for convenience. 

\onequery*

\begin{proof}
    For any sample $S \in \mcX^n$, define
    \begin{align*}
        f(S) \coloneqq \sup_{\violet{\wh{\mcU(S)}} \text{ is $\eta$-close to }\mcU(S)} \left\{ \Psi(\violet{\wh{\mcU(S)}})\right\}.
    \end{align*}
    We will show that $f$ satisfies the $c \coloneqq \frac{2\ell}{n}$-bounded difference property. Consider any samples $S, S' \in \mcX^n$ that differ in only a single point. We will show that for every point in the set \[\{\Psi(\violet{\wh{\mcU(S)}}) \mid \violet{\wh{\mcU(S)}} \text{ is $\eta$-close to }\mcU(S) \},\]  there is some point in the set \[ \{\Psi(\violet{\wh{\mcU(S')}}) \mid \violet{\wh{\mcU(S')}} \text{ is $\eta$-close to }\mcU(S') \}\]  that differs from it by at most $\pm \frac{2\ell}{n}$, \violet{and vice versa}. This implies that $f$ satisfies the $c \coloneqq \frac{2\ell}{n}$-bounded difference property.
    
    As $S$ and $S'$ only differ in a single piont,
    \begin{align*}
        \TV(\mcU(S), \mcU(S')) \leq \frac{1}{n}.
    \end{align*}
    Furthermore by \Cref{def:local}, for any $\violet{\wh{\mcU(S)}}$ that is $\eta$-close to $\mcU(S)$, there is some $\violet{\wh{\mcU(S')}}$ that is $\eta$-close to $\mcU(S')$ satisfying
    \begin{align*}
        \TV(\violet{\wh{\mcU(S)}}, \violet{\wh{\mcU(S')}}) \leq \frac{\ell}{n}.
    \end{align*}
    Using the \Cref{def:tv-distance} and the fact that the range of $\Psi$ is a length-$2$ interval, the above implies that
    \begin{align*}
        \left|\Psi(\violet{\wh{\mcU(S)}}) - \Psi(\violet{\wh{\mcU(S')}})\right| \leq \frac{2\ell}{n}
    \end{align*}
    proving that $f$ satisfies the $\frac{2\ell}{n}$-bounded difference property. By McDiarmid's inequality, $f(\bS)$ concentrates around its mean. Lastly, we show that
    \begin{align}
        \label{eq:oblivious-expectation}
         \Ex_{\bS \sim \mcD^n}[f(\bS)] \leq T.
    \end{align}
    Suppose for the sake of contradiction that $\Ex[f(\bS)] > T+\eps$ for some $\eps > 0$. For each $S \in \mathcal{X}^n$, let $\violet{\wh{\mcU(S)}}$ be $\eta$-close to $\mcU(S)$ and satisfy $\Psi(\violet{\wh{\mcU(S)}}) \geq f(S) - \eps$ (which exists by the definition of $f$). We can define $\wh{\mcD}$ as the distribution where, to sample $\bx \sim \wh{\mcD}$, we
    \begin{enumerate}
        \item Draw an i.i.d.~sample $\bS \sim \mcD^n$.
        \item Draw $\bx \sim \violet{\wh{\mcU(\bS)}}$ uniformly.
    \end{enumerate}
    By \Cref{fact:sample-to-dist}, $\wh{\mcD}$ is $\eta$-close to $\mcD$. Then,
    \[ 
        T \geq \Psi(\wh{\mcD}) 
        = \Ex_{\bS \sim \mcD^n}[\Psi(\violet{\wh{\mcU(\bS)}})] 
         \geq \Ex_{\bS \sim \mcD^n}[f(\bS) - \eps]
        = \Ex[f(\bS)]- \eps > T. 
\]     This is a contradiction, so \Cref{eq:oblivious-expectation} holds. \Cref{lem:one-query} follows from McDiarmid's inequality applied to $f$.
\end{proof}


\duality*
\begin{proof} 
    We'll actually prove a slightly more general result. We'll show that for any any \emph{distribution} $\mcE$ satisfying $\phi^{(i)}(\mcE) \in v_i \pm \frac{\tau}{2}$ for each $i \in [k]$, it is also true that $\Psi(\mcE) \geq T + \frac{\tau}{2}$. \Cref{lem:many-sqs-to-one} follows by setting $\mcE = \mcU(\wh{S})$.

    We define $A \in \R^d$ to be
    \begin{align*}
        A \coloneqq \left\{\left(\phi^{(1)}(\wh{\mcD}), \ldots, \phi^{(k)}(\wh{\mcD})\right) \mid \wh{\mcD} \text{ is $\eta$-close to }\mcD\right\},
    \end{align*}
    which is convex since the cost function is closed under mixtures (\Cref{def:closed-under-mixtures}). 
    We define $B$ to be
    \begin{align*}
        B \coloneqq \left\{b \in \R^k \mid b_i \in v_i \pm \tau \text{ for all }i \in [k]\right\},
    \end{align*}
    which is convex since it is the intersection of halfspaces. By the assumptions of \Cref{lem:many-sqs-to-one}, $A$ and $B$ are disjoint. Let $w \in \R^k$ and $T \in R$ be the vector and threshold respectively guaranteed to exist by the separating hyperplane theorem, normalized so that $\|w\|_1 = 1$. We define
    \begin{align*}
        \Psi(x) \coloneqq \sum_{i \in [k]}w_i \cdot \phi^{(i)}(x).
    \end{align*}
    As $\|w\|_1 = 1$ and the range of each $\phi^{(i)}(x) \in [-1,1]$ for each $x \in \mcX$, it is also true that $\Psi(x) \in [-1,1]$. We show that $\Psi, T$ meet the two criteria of \Cref{lem:many-sqs-to-one}. The \hyperref[criteria:oblivious-small]{first criteria} holds by the separating hyperplane theorem. The \hyperref[criteria:adaptive-large]{second criteria} is equivalent to showing that any $b \in B_{\mathrm{inner}}$ satisfies $b \cdot w \geq T +\lfrac{\tau}{2}$ where
    \begin{align*}
        B_{\mathrm{inner}} \coloneqq \left\{b \in \R^k \mid b_i \in v_i \pm \lfrac{\tau}{2} \text{ for all }i \in [k]\right\}.
    \end{align*}
    There must be a minimal point for $b \cdot w$ at a ``corner" of $B$. Let $b^\star$ be such a minimal point (i.e. $(b^\star)_i = v_i + c_i \cdot \tau$ for $c_i \in \bits$). For any $b \in B_{\mathrm{inner}}$
    \begin{align*}
        b\cdot w &= b^\star \cdot w + \sum_{i \in [k]} w_i (b_i - b^\star_i) \\
        &= b^\star \cdot w + \sum_{i \in [k]} |w_i| \cdot |b_i - b^\star_i| \tag{$b^\star$ minimal for $b\cdot w$ over $b \in B$}\\
        &\geq b^\star \cdot w + \sum_{i \in [k]} |w_i| \cdot \lfrac{\tau}{2} \tag{$|b_i - b^\star_i| \geq \tau/2$ for $b \in B_{\mathrm{inner}}$}\\
        &\geq T + \lfrac{\tau}{2}. \tag{Separating hyperplane theorem, $\|w\|_1 = 1$}
    \end{align*}
    This implies the \hyperref[criteria:adaptive-large]{second criteria} of \Cref{lem:many-sqs-to-one}.
\end{proof}

\violet{
\begin{remark}
    \label{remark:SQ-strong}
    \Cref{thm:sq-formal} also holds with strong adaptive adversaries (\Cref{def:strong-adaptive} in \Cref{sec:technical remarks}) rather than just adaptive adversaries, with the appropriate changes in constants. The proof only differs in \Cref{lem:one-query}. We then wish to bound the probability that the adversary can make $\Psi(\hat{\bS}') \geq T + \tau/2$, where $\hat{\bS}$ and $\hat{\bS}'$ are as in \Cref{def:strong-adaptive}. For constant $\ell$,
    \begin{align*}
        \Prx[\Psi(\hat{\bS}') \geq T + \tau/2] &\leq  \Prx_{\hat{\bS}, \hat{\bS}'}\left[|\Psi(\hat{\bS}) - \Psi(\hat{\bS}')| \geq \tau/4\right] + \Prx_{\hat{\bS}}[\Psi(\wh{\bS}) \geq T + \tau/4] \\
        &\leq  \Prx_{\hat{\bS}, \hat{\bS}'}\left[\TV(\mcU(\hat{\bS}), \mcU(\hat{\bS}')) \geq \tau/8\right] + \exp(-O(\tau^2n)) \tag{\Cref{lem:one-query}} \\
        & \leq \exp(-O(\tau^2n)). \tag{\Cref{def:strong-adaptive}}
    \end{align*}
    Once \Cref{lem:one-query} is modified to handle strong adaptive adversaries, the remainder of the proof of \Cref{thm:sq-formal} applies unchanged.
\end{remark}

}

%% file: IfItCanBeDone.tex
\section{Proof of~\Cref{thm:if-can-be-done}: If adaptivity can be neutralized, subsampling does it}

\violet{
Towards tackling \Cref{question} for all algorithms, we define the subsampling filter, a natural ``wrapper algorithm" that operates only on samples and can be applied to any existing algorithm:

\begin{definition}[Subsampling Filter]
 Define the subsampling filter $\Phi_{m\to n} : \mcX^m \to \mcX^n$ that, given a set $S \in\mathcal{X}^m$, subsamples $n$ elements $\bS'\sim \mathcal{U}(S)^n$ and returns them. We will also write $\Phi_{*\to n}$ when the size of $S$ is variable. 
\end{definition}

Intuitively, by requesting a large number of points $m$ and randomly subsampling points for the original algorithm, the filter should able to neutralise some of the power of the adaptive adversary, since the adversary cannot know which subsample the algorithm will receive.  In this section we will prove \Cref{thm:if-can-be-done} which, informally speaking, states that {\sl if} the noise model is such that adaptivity can be neutralised, then subsampling does it. In the next section, we carry out this proof strategy for the specific case of additive noise and establish~\Cref{thm:add}. 

}

\subsection{Definitions and the formal statement of~\Cref{thm:if-can-be-done}}

\newcommand{\Adaptivemax}{\mathrm{Adaptive\text{-}Max}}

\newcommand{\Adaptivemin}{\mathrm{Adaptive\text{-}Min}}

\newcommand{\Strongmax}{\mathrm{Strong\text{-}Adaptive\text{-}Max}}

\newcommand{\Strongmin}{\mathrm{Strong\text{-}Adaptive\text{-}Min}}

\newcommand{\Obliviousmax}{\mathrm{Oblivious\text{-}Max}}

\newcommand{\Obliviousmin}{\mathrm{Oblivious\text{-}Min}}

\newcommand{\Estmax}{\mathrm{Est\text{-}Max}}

\newcommand{\Estmin}{\mathrm{Est\text{-}Min}}

We begin by formalizing what it means for the behavior of an algorithm $\mcA$ in the presence of an oblivious adversary to be equivalent to that of an algorithm $\mcA'$ in the presence of an adaptive adversary. Roughly speaking, that corresponds to the range of acceptance probabilities $\mcA$ can have with all possible oblivious adversaries being close to the range of acceptance probabilities that $\mcA'$ can have with all possible adaptive adversaries.

 \begin{definition} %
    \label{def:Amax}
Fix a cost function $\rho$, budget $\eta \ge 0$, and distribution $\mathcal{D}$ over $\mathcal{X}$.  For an algorithm $\mcA:\mcX^* \to \zo$, we define: 
    \begin{align*}
        \Obliviousmax_{\rho,\eta}(\mathcal{A},\mathcal{D},n) \coloneqq \sup_{\rho(\mcD, \wh{\mcD}) \leq \eta}\left\{ \Ex_{\bS \sim \wh{\mcD}^n}[\mcA(\bS)]\right\}, \\
        \Obliviousmin_{\rho,\eta}(\mathcal{A},\mathcal{D},n) \coloneqq \inf_{\rho(\mcD, \wh{\mcD}) \leq \eta}\left\{\Ex_{\bS \sim \wh{\mcD}^n}[\mcA(\bS)]\right\},
    \end{align*}
the maximum and minimum acceptance probabilities of $\mcA$ given an obliviously corrupted $\mcD$.  We similarly define the adaptive versions:      \begin{align*}
        \Adaptivemax_{\rho,\eta}(\mathcal{A},\mathcal{D},m) \coloneqq \Ex_{\bS \sim \mcD^m} \Bigg[\sup_{\rho(\mcU(\bS), \mcU(\wh{\bS})) \leq \eta} \left\{ \mcA(\wh{\bS}) \right\} \Bigg], \\
        \Adaptivemin_{\rho,\eta}(\mathcal{A},\mathcal{D},m) \coloneqq \Ex_{\bS \sim \mcD^m} \Bigg[\inf_{\rho(\mcU(\bS), \mcU(\wh{\bS})) \leq \eta} \left\{ \mcA(\wh{\bS}) \right\} \Bigg].
    \end{align*}
\end{definition}

\begin{definition}[$\eps$-equivalent]
    \label{def:equivalent}
    Fix a cost function $\rho$ and a budget $\eta \ge 0$.  Let  $\mathcal{A},\mathcal{A}' :\mathcal{X}^* \to \zo$  be two algorithms.  We say that $\mathcal{A}$ in the presence of $(\rho,\eta)$-oblivious adversaries is {\sl $(n,m,\eps)$-equivalent} to $\mathcal{A}'$ in the presence of $(\rho,\eta)$-adaptive adversaries if the following holds \violet{for all distributions $\mathcal{D}$ over $\mathcal{X}$}: 
\[ \Adaptivemax_{\rho,\eta}(\mathcal{A}',\mathcal{D},m) = \Obliviousmax_{\rho,\eta}(\mathcal{A},\mathcal{D},n) \pm \eps, \]
and likewise for $\mathrm{Min}$ instead of $\mathrm{Max}$. If the algorithms $\mcA$ or $\mcA'$ are randomized, then these expectations are also over the randomness of the algorithms.
\end{definition}

We now state the formal version of \Cref{thm:if-can-be-done}: 
\begin{theorem}[Formal version of \Cref{thm:if-can-be-done}: If it is possible, subsampling does it] 
    \label{thm:if-can-be-done-formal} 
    Fix a cost function $\rho$ and budget $\eta \ge 0$.  Suppose that $\mathcal{A}$ and $\mathcal{A}'$ are algorithms where $\mathcal{A}$ in the presence of $(\rho,\eta)$-oblivious adversaries is $(n,m,\eps)$-equivalent to $\mathcal{A}'$ in the presence of $(\rho,\eta)$-adaptive adversaries.

    Consider the subsampling algorithm $\mathcal{A}_{\mathrm{sub}}\coloneqq \mathcal{A}  \circ \Phi_{*\to n}$ which, given a sample $S \in\mathcal{X}^*$, subsamples $n$ elements $\bS'\sim \mathcal{U}(S)^n$ and returns $\mathcal{A}(\bS')$.  For 
    \[ M \coloneqq O\left(\frac{m^2\log(1/\eps)^2}{\eps^5}\right),\]
   we have that $\mathcal{A}$ in the presence of $(\rho,\eta)$-oblivious adversaries is $(n,M,9\eps)$-equivalent to $\mathcal{A}_{\mathrm{sub}}$ in the presence of $(\rho,\eta)$-adaptive adversaries. 
\end{theorem} 

\begin{remark}[Search to decision reduction]
    In both \Cref{thm:add,thm:if-can-be-done}, we focus on decision algorithms that output a single bit $\zo$, rather than the more general setting of search algorithms that output an answer from some set $\mcY$. This is without loss of generality. Given a search algorithm $\mcA: \mcX^* \to \mcY$, and any set of ``good outputs" $Y \subseteq \mcY$, we could define an algorithm $\mcB \coloneqq \Ind_Y \circ \mcA$ where $\mcB(S) = 1$ iff $\mcA(S) \in Y$. Then, we can directly apply \Cref{thm:add,thm:if-can-be-done} to $\mcB$. Hence, \Cref{thm:add,thm:if-can-be-done} hold for search algorithms with the appropriate definition of ``equivalence" for search algorithms: We say $\mcA:\mcX^\star \to \mcY$ and $\mcA':\mcX^\star \to \mcY$ are $\eps$-equivalent if for every $Y \subseteq \mcY$, $\Ind_Y \circ \mcA$ and $\Ind_Y \circ \mcA'$ are $\eps$-equivalent (according to \Cref{def:equivalent}).
\end{remark}


\subsection{Proof of~\Cref{thm:if-can-be-done-formal}}

Our proof of \Cref{thm:if-can-be-done-formal} relies on the following simple lemma. Roughly speaking, it states that sampling with replacement and sampling without replacement are nearly indistinguishable when the population is a quadratic factor larger than the number of samples.
\begin{lemma}
    \label{lem:sample-small-tv}
    For any distribution $\mcD$ and integers $m,M \in \N$, let $\Phi_{M\to m}\circ \mathcal{D}^M$ be the distribution with the following generative process: first draw a size-$M$ sample $\bS \sim \mathcal{D}^M$, and then subsample, with replacement, $m$ points from $\bS$.  Then 
    \[ \TV(\mathcal{D}^m,\Phi_{M\to m}\circ \mathcal{D}^M)\le \frac{\binom{m}{2}}{M}.\] 
    \end{lemma}


\begin{proof} 
We describe a coupling of $\bS \sim \mcD^m$ and $\bS' \sim \Phi_{M\to m}\circ \mathcal{D}^M$  such that $\Pr[\bS \neq \bS'] \leq \binom{m}{2}/M$. \begin{enumerate}
        \item Initialize $S$ and $S'$ to be empty sets, and $y_1, \ldots, x_M$ to be unset variables. 
        \item Repeat $m$ times:
        \begin{enumerate}
            \item Draw $\bi \sim [M]$ uniformly.
            \item If $y_{\bi}$ is unset, draw $\bx \sim \mcD$ and set $y_{\bi} \leftarrow \bx$. Then, add $\bx$ to both $\bS$ and $\bS'$.
            \item Otherwise, add $y_{\bi}$ to $\bS'$ and sample $\bx \sim \mcD$ to add to $\bS$.
        \end{enumerate}
    \end{enumerate}
    It is straightforward to verify that the above generative process leads to the distribution of $\bS$ and $\bS'$ being that of $\mcD^m$ and $\Phi_{M\to m} \circ \mcD^M$ respectively. Furthermore, if $\bS \neq \bS'$, that means there is some index $i \in [M]$ that was sampled at least twice. If we fix $j_1 \neq j_2 \in [m]$ and some $i \in [M]$, the probability $i$ is the index chosen at steps $j_1$ and $j_2$ is $1/M^{2}$. Union bounding over the $M$ choices for $i$ and $\binom{m}{2}$ for $j_1, j_2$ gives that
\[         \Pr[\bS \neq \bS'] \leq \frac{\binom{m}{2}}{M}.
   \] 
    The desired result follows from the definition of total variation cost, \Cref{def:tv-distance}.
\end{proof} 


\Cref{lem:sample-small-tv} will be used in conjunction with the following fact: 

\begin{fact}
    \label{fact:tv-test}
    Suppose there exists a test which, given $c$ samples from a distribution $\mathcal{E}$ that is either $\mcD_0$ or $\mcD_1$, returns $0$ if $\mathcal{E} = \mcD_0$ with probability at least $\frac{3}{4}$, and returns $1$ if $\mathcal{E} = \mcD_1$ with probability at least $\frac{3}{4}$. Then,
    \[ 
        \TV(\mcD_1, \mcD_2) \geq \frac{1}{2c}.
    \] 
\end{fact}

Together, \Cref{lem:sample-small-tv} and~\Cref{fact:tv-test} imply that for appropriately chosen $m$ and $M$, there is no sample-efficient test distinguishing $\mcD^m$ from $\Phi_{M\to m}\circ \mcD_M$ with high probability. We will use this to prove \Cref{thm:if-can-be-done-formal} by contradiction. 
Using the assumption that $\mathcal{A}'$ is equivalent to $\mathcal{A}$, we will design a sample-efficient test that approximates $\Obliviousmax_{\rho,\eta}$ and $\Obliviousmin_{\rho,\eta}$ for $\mathcal{A}$ with respect to both $\mathcal{D}^m$ and $\Phi_{M\to m}\circ \mathcal{D}^M$.  We then show that if $\mathcal{A}_{\mathrm{sub}} \coloneqq \Phi_{*\to n}\circ \mathcal{A}$ is {\sl not} equivalent to $\mathcal{A}$, then these values will distinguish the two distributions. 

The following lemma carries out the first part of this plan: 

\begin{lemma}
    \label{lem:estimators-from-equivalent}
    Let $\mathcal{A}$ and $\mathcal{A}'$ be as in~\Cref{thm:if-can-be-done-formal}. There is an estimator $\Estmax_{\rho,\eta}$ that uses 
\[         m' \coloneqq \frac{m \log(2/\eps)}{2\eps^2}
   \] 
     samples from a distribution $\mathcal{D}$ and returns an estimate of $\Obliviousmax_{\rho,\eta}(\mathcal{A},\mathcal{D},n)$ that is accurate to $\pm 2\eps$ with probability at least $1 - \eps$, and likewise an estimator $\Estmin_{\rho,\eta}$ for $\Obliviousmin_{\rho,\eta}$.   Formally, for all distributions $\mcD$ over  $\mcX$,
     \[ \Prx_{\bS\sim \mathcal{D}^{m'}}\Big[\Estmax_{\rho,\eta}(\bS) = \Obliviousmax_{\rho,\eta}(\mathcal{A},\mathcal{D},n) \pm 2\eps\Big] \ge 1-\eps,  \]
     and likewise for $\Estmin_{\rho,\eta}$ and $\Obliviousmin_{\rho,\eta}$.
\end{lemma}
\begin{proof}
We will prove the lemma for $\Estmax_{\rho,\eta}$ and $\Obliviousmax_{\rho,\eta}$; the proof for Min instead of Max is identical.  $\Estmax_{\rho,\eta}$ computes an estimate satisfying: 
\begin{equation}  \Prx_{\bS\sim \mathcal{D}^{m'}}\Big[ \Estmax_{\rho,\eta}(\bS) = \Adaptivemax_{\rho,\eta}(\mathcal{A'},\mathcal{D},m)\pm \eps \Big] \le \eps. \label{eq:avg-estimate}
\end{equation} 
This is sufficient to guarantee that $\Estmax_{\rho,\eta}$'s estimate is within $\pm 2\eps$ of $\Obliviousmax_{\rho,\eta}(\mathcal{A},\mathcal{D},n)$ by our assumption that $\mathcal{A}$ is $(n,m,\eps)$-equivalent to $\mathcal{A}'$.  To provide such an estimate, $\Estmax_{\rho,\eta}(\bS)$ draws $\log(2/\eps)/(2\eps^2)$ many size-$m$ samples $\bS' \sim \mcD^m$. For each, it computes: 
\begin{equation}
    \label{eq:compute-sup}
    \sup_{\rho(\wh{\bS'},\bS')\le \eta} \big\{ \mathcal{A}'(\wh{\bS'})\big\}
\end{equation} 
and returns the average of these supremums.  By the Chernoff bound, this average satisfies \Cref{eq:avg-estimate}. 
\end{proof}

\begin{remark} 
We are only concerned with the sample efficiency of these estimators, not their time efficiency or even whether they are computable. Indeed, as stated, an algorithm computing the estimators would need to loop or infinitely many $\wh{\bS'} \in \mcX^*$ to compute \Cref{eq:compute-sup}. For us they are just an analytical tool used to prove~\Cref{thm:if-can-be-done-formal}, the conclusion of which gives an algorithm $\mcA_{\mathrm{sub}} \coloneqq \mathcal{A}\circ \Phi_{*\to n}$ that inherits the time efficiency of $\mcA$.
\end{remark} 

The next lemma notes that the  $\Adaptivemax_{\rho,\eta}$ of $\mathcal{A}_{\mathrm{sub}} \coloneqq \mathcal{A}\circ \Phi_{*\to n}$ can be expressed in terms of the $\Obliviousmax_{\rho,\eta}$ of $\mathcal{A}$.  Formally:  

\begin{lemma}
    \label{cor:amax-sufficient}
$\ds \Adaptivemax_{\rho,\eta}(\mathcal{A}_{\mathrm{sub}},\mathcal{D},M) = \Ex_{\bS\sim \mathcal{D}^M}\big[ \Obliviousmax_{\rho,\eta}(\mathcal{A},\mathcal{U}(\bS),n)\big],$
    and likewise for $\mathrm{Min}$ instead of $\mathrm{Max}$.
\end{lemma}

\begin{proof}
The lemma follows from this series of identities: 
\begin{align*} 
\Adaptivemax_{\rho,\eta}(\mathcal{A}_\mathrm{sub},\mathcal{D},M) &= \Ex_{\bS \sim \mcD^M} \Bigg[\sup_{\rho(\mcU(\bS), \mcU(\wh{\bS})) \leq \eta} \left\{ \mcA_{\mathrm{sub}}(\wh{\bS}) \right\} \Bigg] \tag{Definition of $\Adaptivemax_{\rho,\eta}$}\\
&= \Ex_{\bS \sim \mcD^M} \Bigg[\sup_{\rho(\mcU(\bS), \mcU(\wh{\bS})) \leq \eta} \left\{ (\mathcal{A}\circ \Phi_{*\to n})(\wh{\bS}) \right\} \Bigg] \tag{Definition of $\mathcal{A}_{\mathrm{sub}}$} \\
&= \Ex_{\bS \sim \mcD^M} \Bigg[\sup_{\rho(\mcU(\bS), \mathcal{E}) \leq \eta} \left\{  \Ex_{\bS'\sim \mathcal{E}^n}[\mathcal{A}(\bS')] \right\}  \Bigg] \tag{Definition of the subsampling filter $\Phi_{*\to n}$}\\
&= \Ex_{\bS\sim \mathcal{D}^M}\big[ \Obliviousmax_{\rho,\eta}(\mathcal{A},\mathcal{U}(\bS),n)\big], 
\tag{Definition of $\Obliviousmax_{\rho,\eta}$}
\end{align*} 
where the penultimate identity also uses our convention that distributions have rational weights, and therefore can expressed as the uniform distribution over a sufficiently large multiset of elements.
\end{proof}

The following lemma completes the proof of \Cref{thm:if-can-be-done-formal}.

\begin{lemma}
\label{lem:estimator-to-subsample}
Let $m'$ be as in~\Cref{lem:estimators-from-equivalent} and define $M\coloneqq 14(m')^2/\eps$.  Then 
\[ \Ex_{\bS\sim \mathcal{D}^M}\big[ \Obliviousmax_{\rho,\eta}(\mathcal{A},\mathcal{U}(\bS),n)\big] = \Obliviousmax_{\rho,\eta}(\mathcal{A},\mathcal{D},n)\pm 9\eps, \] 
and likewise for $\mathrm{Min}$ instead of $\mathrm{Max}$.
\end{lemma}

\begin{proof}
Our proof proceed by contradiction: assuming that $\Ex_{\bS\sim \mathcal{D}^M}\big[ \Obliviousmax_{\rho,\eta}(\mathcal{A},\mathcal{U}(\bS),n)\big]$ is more than $9\eps$ far from $ \Obliviousmax_{\rho,\eta}(\mathcal{A},\mathcal{D},n)$, we will prove that
    \begin{align}
        \label{eq:tv-large}
         \TV(\mcD^{m'}, \Phi_{M\to m'}\circ \mathcal{D}^M) \geq \frac{\eps}{14}.
    \end{align}
    From \Cref{lem:sample-small-tv}, we know that
\[ 
         \TV(\mcD^{m'}, \Phi_{M\to m'}\circ \mathcal{D}^M) < \frac{(m')^2}{M} = \frac{\eps}{14},
    \] 
which yields the desired contradiction.  To establish~\Cref{eq:tv-large}, we design an algorithm that given $\ceil{\frac{6}{\eps}} \leq \frac{7}{\eps}$ samples from either $\mcD^{m'}$ from $\Phi_{M\to m'}\circ \mathcal{D}^M$ is able to distinguish them with probability $\frac{3}{4}$. Once we do, the desired result follows from \Cref{fact:tv-test}.
    
    Let $\mu \coloneqq \Obliviousmax_{\rho,\eta}(\mathcal{A},\mathcal{D},n)$. First, we define
    \begin{align*}
        \delta \coloneqq \Prx_{\bS \sim \mcD^{M}}\big[|\Obliviousmax_{\rho,\eta}(\mathcal{A},\mathcal{U}(\bS),n) - \mu| >  4\eps\big]
    \end{align*}
    and we bound 
    \begin{align*}
        \Ex_{\bS \sim \mcD^M}\big[\Obliviousmax_{\rho,\eta}(\mathcal{A},\mathcal{U}(\bS),n)\big] \leq (1-\delta) \cdot (\mu + 4\eps) + \delta \cdot 1 \leq \mu + 4\eps + \delta.
    \end{align*}
    Similarly, 
    \begin{align*}
        \Ex_{\bS \sim \mcD^M}\big[\Obliviousmax_{\rho,\eta}(\mathcal{A},\mathcal{U}(\bS),n)\big] &\geq (1-\delta) \cdot (\mu - 4\eps) + \delta \cdot 0 \\
        &\geq \mu - 4\eps - \delta. \tag{$\mu - 4\eps \leq 1$}
    \end{align*}
    Hence by our assumption that
    \begin{align*}
        \left|\Ex_{\bS \sim \mcD^M}\big[\Obliviousmax_{\rho,\eta}(\mathcal{A},\mathcal{U}(\bS),n)\big] - \Obliviousmax_{\rho,\eta}(\mathcal{A},\mathcal{D},n) \right| > 9\eps,
    \end{align*}
    we can conclude that $\delta > 5\eps$.
    
    Now let $\mathcal{E}$ be either $\mcD^{m'}$ or $\Phi_{M\to m'}\circ \mathcal{D}^M$. Our test to determine which $\mathcal{E}$ is will do the following: draw $\ceil{\frac{6}{\eps}}$ samples from $\mcE$, $\bS \sim \mathcal{E}$, and run $\Estmax_{\rho,\eta}(\bS)$ on each. If less than a $2\eps$ fraction of the estimates returned by $\Estmax_{\rho,\eta}(\bS)$ differ from $\mu$ in more than $\pm 2\eps$, return that $\mathcal{E} = \mcD^{m'}$. Otherwise, return that $\mathcal{E} = \Phi_{M\to m'}\circ \mathcal{D}^M$. We will prove this test succeeds with probability at least $1 - e^{-2} \geq \frac{3}{4}$.   We consider the two possible cases: \begin{enumerate} 
    \item  Case 1: $\mathcal{E}= \mcD^{m'}$. In this case, given a sample from $\mathcal{E}$, $\Estmax_{\rho,\eta}(\bS)$ returns an estimate that is within $\pm 2\eps$ of $\mu$ with probability at least $1 - \eps$. By the Chernoff bound, given $\frac{6}{\eps}$ such samples, the probability more than $2\eps$ fraction deviate from $\mu$ by more than $\pm \eps$ is at most $\exp(-\frac{1}{3} \cdot \frac{6}{\eps} \cdot \eps) = e^{-2}$. Therefore, the test succeeds with probability at least $1 - e^{-2}$.
    \item Case 2: $\mathcal{E} = \Phi_{M\to m'}\circ \mathcal{D}^M$. We showed above that with probability at least $5\eps$ over a sample $\bS \sim \mathcal{D}^M$ we have $\left|\Obliviousmax_{\rho,\eta}(\mathcal{A},\mathcal{U}(\bS),n) - \mu\right| > 4 \eps$. When that's the case, $\Estmax_{\rho,\eta}(\bS)$ returns an estimate that is further than $\pm 2\eps$ from $\mu$ with probability at least $1 - \eps$. Therefore, on a single sample, the probability that the estimate of $\mcA$ deviates from $\mu$ by more than $\pm 2\eps$ is at least $5\eps(1-\eps) \geq 4\eps$. By the Chernoff bound, given $\frac{6}{\eps}$ samples, the probability that at most $2\eps$ fraction deviate $\mu$ by at most $\pm 2\eps$ is at most $\exp(-\frac{1}{8} \cdot \frac{6}{\eps} \cdot 4\eps) = e^{-3}$. Therefore, the test succeeds with probability at least $1 - e^{-2}$.
 
    \end{enumerate} 
    Hence, given $\ceil{\frac{6}{\eps}}$ samples, it is possible to distinguish $\mcD^{m'}$ from $\Phi_{M\to m'}\circ \mathcal{D}^M$ with a success probability of at least $\frac{3}{4}$. \Cref{eq:tv-large} follows from \Cref{fact:tv-test}, completing the proof by contradiction.
\end{proof}
\violet{
\begin{remark}[Strong adaptive adversaries]
    \label{remark:strong-adaptive-subsampling}
    If $\mcA:\mcX^n \to \zo$ in the presence of oblivious adversaries is $(n,M,\eps)$-equivalent to $\mcA_{\mathrm{sub}} \coloneqq \mcA \circ \Phi_{* \to n}$ in the presence of adaptive adversaries, then it is also $(n, M, 2\eps)$-equivalent to $\mcA_{\mathrm{sub}}$ in the presence of \emph{strong adaptive adversaries} (\Cref{def:strong-adaptive} in \Cref{sec:technical remarks}) as long as $M = \Omega(n^2/\eps^2)$. This applies to both \Cref{thm:if-can-be-done-formal} in this section and \Cref{thm:add-formal} in the next.
    
    Let $\hat{\bS}$ and $\hat{\bS}'$ be defined as in \Cref{def:strong-adaptive}. Then, for any strong adaptive adversary supplying the sample $\hat{\bS}'$, there is some adaptive adversary supplying the sample $\hat{\bS}$, so we wish to compare $\Ex[\mcA_{\mathrm{sub}}(\hat{\bS}')]$ to $\Ex[\mcA_{\mathrm{sub}}(\hat{\bS})]$. Recall that $\mcA_{\mathrm{sub}}$ first subsamples to $n$ points, so
    \begin{align*}
        \left|\Ex[\mcA_{\mathrm{sub}}(\hat{\bS}')] - \Ex[\mcA_{\mathrm{sub}}(\hat{\bS})]\right| &\leq n \Ex_{\hat{\bS}, \hat{\bS}'}[\TV(\mcU(\hat{\bS}), \mcU(\hat{\bS}'))] \\
        &\leq O(n/\sqrt{M}) = \eps. \tag{by \Cref{eq:strong-adaptive-sqrt}}
    \end{align*}
\end{remark}
}

%% file: AdditiveSubsampling.tex
\section{\violet{Proof of~\Cref{thm:add}: The subsampling filter neutralizes adaptive additive noise}}
In this section, we prove the following theorem: 
\begin{theorem}[Formal version of \Cref{thm:add}]
    \label{thm:add-formal} 
    
     Fix a budget $\eta \ge 0$, distribution $\mathcal{D}$ over $\mathcal{X}$, and algorithm $\mathcal{A}: \mcX^n \to \zo$. Consider the subsampling algorithm $\mathcal{A}_{\mathrm{sub}}\coloneqq \mathcal{A}\circ\Phi_{*\to n}$ which, given a sample $S \in\mathcal{X}^*$, subsamples $n$ elements $\bS'\sim \mathcal{U}(S)^n$ and returns $\mathcal{A}(\bS')$.  For 
    \[ M \coloneqq O\left(\frac{n^4\log(|\mcX|)}{\eps^2}\right),\]
    we have that $\mathcal{A}$ in the presence of $(\dadd,\eta)$-oblivious adversaries is $(n,M,\eps)$-equivalent to $\mathcal{A}_{\mathrm{sub}}$ in the presence of $(\dadd,\eta)$-adaptive adversaries. 
\end{theorem}

Implicit in the proof of \Cref{thm:if-can-be-done-formal}, we proved the following.
\begin{lemma}
    \label{lem:est-to-equivalent-weaker}
    Fix a cost function $\rho$, budget $\eta \geq 0$, and $\eps > 0$. Suppose that for an algorithm $\mcA: \mcX^n \to \zo$, there are estimators $\Estmax_{\rho,\eta},\Estmin_{\rho,\eta}:\mcX^{m} \to \zo$ that use $m$ samples from a distribution $\mathcal{D}$ and returns estimates satisfying
     \begin{align*}
         \Prx_{\bS\sim \mathcal{D}^{m'}}\Big[\Estmax_{\rho,\eta}(\bS) = \Obliviousmax_{\rho,\eta}(\mathcal{A},\mathcal{D},n) \pm 2\eps\Big] \ge 1-\eps,
     \end{align*}
     and likewise for $\mathrm{Min}$ instead of $\mathrm{Max}$. Then, for
     \begin{align*}
         M = \frac{14m^2}{\eps},
     \end{align*}
     we have that $\mathcal{A}$ in the presence of $(\rho,\eta)$-oblivious adversaries is $(n,M,9\eps)$-equivalent to  $\mathcal{A}_{\mathrm{sub}}\coloneqq \mathcal{A}\circ\Phi_{*\to n}$ in the presence of $(\rho,\eta)$-adaptive adversaries. 
\end{lemma}
In order to prove \Cref{thm:add-formal}, we'll construct the estimator $\Estmax_{\dadd,\eta}$ (and likewise for Min). The goal is to estimate
\begin{align*}
     \Obliviousmax_{\dadd,\eta}(\mathcal{A},\mathcal{D},n) \coloneqq \sup_{ \wh{\mcD} = (1- \eta)\cdot \mcD + \eta \cdot \mcE}\left\{ \Ex_{\bS \sim \wh{\mcD}^n}[\mcA(\bS)]\right\}.
\end{align*}
The key insight is rather than trying all possible distributions $\mcE$, to compute a $\pm \eps$ approximation, it suffices to consider those $\mcE$ that are equal to $\mcU(T)$ for some $T \in \mcX^{n^2/\eps}$. Our final result will have a logarithmic dependence on the number of $\mcE$ we need to try, which results in just a logarithmic dependence on $|\mcX|$.

The following definition and accompany fact will be useful.

\begin{definition}[Stochastic function]
    For any sets $\mcX, \mathcal{Y}$, a stochastic function $\boldf:\mcX \to \mathcal{Y}$ is a collection of distributions $\{\mcD_x \mid x \in \mcX\}$ each supported on $\mathcal{Y}$ where the notation $\boldf(x)$ indicates an independent draw from $\mcD_x.$
\end{definition}

\begin{fact}
    \label{fact:stochastic-tv}
    For any distribution $\mcD_1, \mcD_2$ supported on a domain $\mcX$, and any stochastic function $\boldf:\mcX \to \mathcal{Y}$, let $\boldf \circ \mcD_i$ be the distribution where to sample $\by \sim \boldf \circ \mcD_i$, we first sample $\bx \sim \mcD_i$ and then sample $\by = \boldf(\bx)$. Then,
    \begin{align*}
        \TV(\boldf \circ \mcD_1, \boldf \circ \mcD_2) \leq \TV(\mcD_1, \mcD_2).
    \end{align*}
\end{fact}
\begin{proof}
    Given a coupling of $\bx_1 \sim \mcD_1$ and $\bx_2 \sim \mcD_2$, consider the coupling of $\by_1 \sim \boldf \circ \mcD_1$ and $\by_2 \sim \boldf \circ \mcD_2$ where if $\bx_1 = \bx_2$ then $\by_1 = \by_2 = \boldf(\bx_1)$ and otherwise $\by_1 = \boldf(\bx_1)$ and $\by_2 = \boldf(\bx_2)$ independently. Then,
\[ 
        \Pr[\by_1 \neq \by_2] \leq \Pr[\bx_1 \neq \bx_2],
   \] 
    implying the desired result by \Cref{def:tv-distance}.
\end{proof}

To prove \Cref{thm:add-formal}, we will design the estimators $\Estmax_{\dadd,\eta}$ and $\Estmin_{\dadd,\eta}$. Plugging the sample complexity of those estimators into \Cref{lem:est-to-equivalent-weaker} would be sufficient for \Cref{thm:add-formal} with the slightly worse $M = n^6 \log(|\mcX|)^2/\eps^7$. To get the optimal $\log(|\mcX|)$ dependence on the domain size (and an improved dependence on $n$ and $\eps$), we need a more refined version of \Cref{lem:est-to-equivalent-weaker} that takes advantage of the structure of the particular estimators we derive. The following Lemma applies for any cost function $\rho$, but we will then design $\mcF$ satisfying \Cref{eq:max-close} specifically for $\rho = \dadd$ in \Cref{lem:design-add-estimator}.

\begin{lemma}
    \label{lem:special-estimator-to-equiv}
    Fix a cost function $\rho$, budget $\eta \geq 0$,  $\eps \in (0, 1]$, and algorithm $\mcA: \mcX^n \to \zo$. Suppose there is a set of stochastic functions $\mcF$ each $\mcX \to \mcX$, that satisfy, for any distribution $\mcD$ over $\mcX$,
     \begin{align}
        \label{eq:max-close}
        \max_{\boldf \in \mcF}\left\{ \Ex_{\bS \sim \mcD^n}[\mcA(\boldf(\bS))]\right\}  = \Obliviousmax_{\rho,\eta}(\mathcal{A},\mcD,n) \pm \eps,
    \end{align}
    and likewise for $\mathrm{Min}$ instead of $\mathrm{Max}$, where $\boldf(S)$ is shorthand for applying $\boldf$ element wise and independently to $S$. Then:
    \begin{enumerate}
        \item There are estimators $\Estmax_{\dadd,\eta}$ and $\Estmin_{\dadd,\eta}$ meeting the requirements of \Cref{lem:est-to-equivalent-weaker} for
        \[
            m' = O\left(\frac{n \log(|\mcF|/\eps)}{\eps^2}\right).
        \]
        In particular, this implies that for
        \[ 
            M = O\left(\frac{n^2 \log(|\mcF|/\eps)^2}{\eps^5}\right),
        \]
        we have that $\mathcal{A}$ in the presence of $(\rho,\eta)$-oblivious adversaries is $(n,M,9\eps)$-equivalent to $\mathcal{A}_{\mathrm{sub}}$ in the presence of $(\rho,\eta)$-adaptive adversaries. 
        \item More directly, for
        \begin{align}
            \label{eq:M-from-max-better}
            M = O\left(\frac{n^2\log(|\mcF|/\eps)}{\eps}\right)
        \end{align}
        we have that $\mathcal{A}$ in the presence of $(\rho,\eta)$-oblivious adversaries is $(n,M,5\eps)$-equivalent to $\mathcal{A}_{\mathrm{sub}}$ in the presence of $(\rho,\eta)$-adaptive adversaries. 
    \end{enumerate}
\end{lemma}

\begin{proof}
    We first give $\Estmax_{\dadd,\eta}$ satisfying the first item. For $r \coloneqq O(\frac{\log(|\mcF|/\eps)}{\eps^2})$, let
    \begin{equation*}
        \Estmax_{\dadd,\eta} = \max_{\boldf \in \mcF} \Bigg\{ \underbrace{\Ex_{\mathrm{trial} \in [r]} \left[\Ex_{\bS \sim \mcD^n}\left[\mcA(\boldf(\bS))\right] \right]}_{\coloneqq \mathrm{Est}^{(\boldf)}}\Bigg\}.
    \end{equation*}
    We note that as long as the samples are reused across different $\boldf \in \mcF$ that $m' = r n$ samples from $\mcD$ suffices to compute the above expression. By Hoeffding's inequality, for any fixed $\boldf \in \mcF$, with probability at least $1 - \frac{\eps}{T}$
    \begin{equation*}
         \mathrm{Est}^{(\boldf)} = \Ex_{\bS \sim \mcD^n}[\mcA(\boldf(\bS))] \pm \eps.
    \end{equation*}
    By union bound, with probability at least $1 - \eps$, the above holds for every $\boldf \in \mcF$. If so, $\Estmax_{\dadd,\eta}$ has the desired accuracy by \Cref{eq:max-close}.\\
    
    Next, we prove the $M$ from \Cref{eq:M-from-max-better} suffices. By \Cref{cor:amax-sufficient}, it suffices to prove that 
    \begin{equation*}
        \Ex_{\bS\sim \mathcal{D}^M}\big[ \Obliviousmax_{\rho,\eta}(\mathcal{A},\mathcal{U}(\bS),n)\big] = \Obliviousmax_{\rho,\eta}(\mathcal{A},\mathcal{D},n)\pm 5 \eps,
    \end{equation*}
    and likewise for $\mathrm{Min}$ instead of $\mathrm{Max}$. Applying \Cref{eq:max-close} to both sides of the above equation, it is sufficient to prove that
    \begin{equation}
        \label{eq:compare-max}
        \Ex_{\bS\sim \mathcal{D}^M}\bigg[\max_{\boldf \in \mcF}\bigg\{ \underbrace{\Ex_{\bS_n \sim \mathcal{U}(\bS)^n}[\mcA(\boldf(\bS_n))]}_{\coloneqq g^{(\boldf)}(\bS)}\bigg\} \bigg]  = \max_{\boldf \in \mcF}\left\{ \Ex_{\bS \sim \mcD^n}[\mcA(\boldf(\bS))]\right\} \pm 3\eps.
    \end{equation}
    Fix a single $\boldf \in \mcF$. We'll use McDiarmid's inequality (\Cref{fact:mcdiarmid}) to say that $g^{(\boldf)}(\bS)$ concentrates around its mean. Take any $S \in \mcX^M$ and suppose we change one point in it to create $S'$. Then $|g^{(\boldf)}(S) - g^{(\boldf)}(S')|$ is at most the probability that the changed point appears in $\bS_n$, which is at most $\frac{n}{M}$. Applying McDiarmid's inequality,
    \begin{align*}
        \Prx_{\bS \sim \mcD^m}\left[g^{(\boldf)}(\bS) = \mu^{(\boldf)} \pm \eps \right] &\geq 1 - 2\exp\left(- \frac{2\eps^2}{(n/M)^2M}\right) \tag*{where $\mu^{(\boldf)} \coloneqq \Ex_{\bS \sim \mcD^M}\left[g^{(\boldf)}(\bS)\right]$} \\
        &\geq 1- \frac{\eps}{|\mcF|}. \tag{using $M = O\left(\frac{n^2\log(|\mcF|/\eps)}{\eps}\right)$}
    \end{align*}
    By union bound, with probability at least $1 - \eps$, for every $\boldf \in \mcF$ we have that $g^{(\boldf)}(\bS) = \mu^{(\boldf)} \pm \eps$ allowing us to bound the left hand side of \Cref{eq:compare-max},
    \begin{align*}
        \Ex_{\bS\sim \mathcal{D}^M}\left[\max_{\boldf \in \mcF} \left\{ g^{(\boldf)}(\bS)\right\} \right] &= \max_{\boldf\in \mcF} \left\{ \mu^{(\boldf)}\right\}  \pm \eps \pm \Prx_{\bS \sim \mcD^m}\left[g^{(\boldf)}(\bS) \neq \mu^{(\boldf)} \pm \eps \text{ for some $\boldf \in \mcF$}\right] \\
        &= \max_{\boldf \in \mcF} \left\{ \mu^{(\boldf)}\right\}  \pm 2\eps.
    \end{align*}
    Lastly, we want to compare the above to the right hand side of \Cref{eq:compare-max} to $\max_{\boldf \in \mcF} \mu^{(\boldf)}$. Fix any $\boldf \in \mcF$. Using the notation of \Cref{lem:sample-small-tv},
    \[
        \mu^{(\boldf)} = \Ex_{\bS \sim \Phi_{M\to n}\circ \mathcal{D}^M} \left[\mcA(\boldf(\bS))\right].
    \]
    Therefore,
    \begin{align*}
        \left|\max_{\boldf \in \mcF} \left\{ \mu^{(\boldf)}\right\}  - \max_{\boldf \in \mcF}\left\{ \Ex_{\bS \sim \mcD^n}[\mcA(\boldf(\bS))]\right\} \right| &\leq  
        \max_{\boldf \in [\mcF]} \left| \mu^{(\boldf)} - \Ex_{\bS \sim \mcD^n}[\mcA(\boldf(\bS))]\right| \\
        &\leq \TV(\mcD^n, \Phi_{M\to n}\circ \mathcal{D}^M)\tag{\Cref{def:tv-distance}} \\
        & \leq \frac{\binom{n}{2}}{M} \leq \eps.\tag{\Cref{lem:sample-small-tv}}
    \end{align*}
    Therefore, the left hand term of \Cref{eq:compare-max} is within $\pm 2\eps$ of $\max_{\boldf \in \mcF} \mu^{(\boldf)}$ and the right hand term is within $\pm \eps$ of $\max_{\boldf \in \mcF} \mu^{(\boldf)}$. Hence, \Cref{eq:compare-max} holds, completing this proof.
\end{proof}

In order to use \Cref{lem:special-estimator-to-equiv}, we need to design $\mcF$ and prove that it satisfies \Cref{eq:max-close} when $\rho = \dadd$.  The below lemma completes the proof of \Cref{thm:add-formal}.

\begin{lemma}
    \label{lem:design-add-estimator}
    Fix a budget $\eta \geq 0$,  $\eps \in (0, 1]$, and algorithm $\mcA: \mcX^n \to \zo$. Let $\mcF$ be the set of $|\mcX|^{n^2/\eps}$ stochastic functions defined below.
    \begin{align*}
        \mcF \coloneqq \{\boldf^{(T)} \mid T \in \mcX^{n^2/\eps}\} \quad \quad \text{where}\quad \boldf^{(T)}(x) \coloneqq \begin{cases}
        x & \text{with probability $1 - \eta$} \\ \by
        \text{ where $\by \sim \mcU(T)$} & \text{with probability $\eta$.}
        \end{cases}
    \end{align*}
    Then for any $\mcD$ over $\mcX$,
     \begin{align*}
        \max_{\boldf \in \mcF}\left\{ \Ex_{\bS \sim \mcD^n}[\mcA(\boldf(\bS))]\right\}  &= \Obliviousmax_{\dadd,\eta}(\mathcal{A},\mcD,n) \pm \eps \\
        &= \sup_{\wh{\mcD} = (1 - \eta)\cdot\mcD + \eta \cdot \mcE}\left\{ \Ex_{\bS \sim \wh{\mcD}^n}[\mcA(\bS)]\right\} \pm \eps,
    \end{align*}
    and likewise for $\mathrm{Min}$ instead of $\mathrm{Max}$.
\end{lemma}
\begin{proof}
    Fix any distribution $\mcD$. First, we note that the distribution of $\boldf^{(T)}(\bx)$ where $\bx \sim \mcD$ is simply that of $(1 - \eta) \cdot \mcD + \eta \cdot \mcU(T)$. Therefore, one direction of the desired result is easy:
    \begin{align*}
        \max_{\boldf \in \mcF}\left\{ \Ex_{\bS \sim \mcD^n}[\mcA(\boldf(\bS))]\right\}  \leq \Obliviousmax_{\dadd,\eta}(\mathcal{A},\mcD,n) .
    \end{align*}
    The remainder of this proof is devoted to proving the left hand side of the above equation is at most $\eps$ smaller than the right hand side. Fix any distribution $\mcE$ and consider $\wh{\mcD} = (1 - \eta)\cdot\mcD + \eta \cdot \mcE$. We'll show that
    \begin{align}
        \label{eq:expected-good}
        \Ex_{\bT \sim \mcE^{n^2/\eps}}\left[\Ex_{\bS \sim \mcD^n}[\mcA(\boldf^{\bT}(\bS))]\right] \geq \Ex_{\bS \sim \wh{\mcD}^n}[\mcA(\bS)] - \eps.
    \end{align}
    In particular, the above implies there is a single choice for $T$ that is within $\eps$ of $\Ex_{\bS \sim \wh{\mcD}^n}[\mcA(\bS)]$. As this holds for all corruptions $\wh{\mcD}$, it implies the desired result.
    
    
    To sample from $\wh{\mcD}^n$ where $\wh{\mcD} =  (1 - \eta) \cdot \mcD + \eta \cdot \mcE$ we can first draw $\bS \sim \mcE^n$ and then return $\bh(\bS)$, applied element wise, where:
    \begin{align*}
        \bh(x) \coloneqq \begin{cases}
        \text{$\by$ where $\by \sim \mcD$} & \text{with probability $1 - \eta$} \\
        \text{$x$} & \text{with probability $\eta$.}
        \end{cases}
    \end{align*}
    Therefore, the distribution of $\boldf^{\bT}(\bS)$ on the left hand side of \Cref{eq:expected-good} is $\bh \circ \Phi_{n^2/\eps \to n} \circ \mcE^{n^2/\eps}$ (using the notation of \Cref{lem:sample-small-tv} and \Cref{fact:stochastic-tv}). Finally, we prove \Cref{eq:expected-good} (recalling that $\wh{\mcD} = (1 - \eta)\cdot\mcD + \eta \cdot \mcE$)
    \begin{align*}
        \left|\Ex_{\bT \sim \mcE^{n^2/\eps}}\left[\Ex_{\bS \sim \mcD^n}[\mcA(\boldf^{\bT}(\bS))]\right] - \Ex_{\bS \sim \wh{\mcD}^n}[\mcA(\bS)]\right| &\leq \TV(\bh \circ \Phi_{n^2/\eps \to n} \circ \mcE^{n^2/\eps}, \bh \circ \mcE^n) \tag{\Cref{def:tv-distance}} \\
        &\leq \TV( \Phi_{n^2/\eps \to n} \circ \mcE^{n^2/\eps}, \mcE^n) \tag{\Cref{fact:stochastic-tv}} \\
        & \leq \eps. \tag{\Cref{lem:sample-small-tv} }
    \end{align*}
\end{proof}

%% file: Subsampling.tex
\section{\violet{Proof of~\Cref{thm:greg}: Lower bounds against the subsampling filter}}
\label{sec:greg}

\violet{
In this section, we show a lower bound on  $m$ needed for the subsampling filter to work.   Our lower bound holds in the setting of additive noise and therefore also shows that the dependence on $|\mathcal{X}|$ in~\Cref{thm:add-formal} is optimal. 
}


\begin{theorem}[Formal version of~\Cref{thm:greg}]
    \label{thm:subsample}
    For any sample size $n$, domain $\mcX$ with $|\mcX| = 2^d$ for an integer $d$, adversary budget $\eta$, and $\eps > 0$, there exists an algorithm $\mcA:\mcX^n \to \zo$ and corresponding subsampled algorithm $\mcA_{\mathrm{sub}} \coloneqq \mcA \circ \Phi_{* \to n}$ for which $\mcA$ in the presence of $(\dadd,\eta)$-oblivious adversaries is \emph{not} $(n, m, 1-\eps)$-equivalent to $\mcA_{\mathrm{sub}}$ in the presence of $(\dadd,\eta)$-adaptive adversaries for any $m =O_\eta(n \log |\mcX| / \log^2 n)$.
\end{theorem}

\paragraph{Proof overview}
Without loss of generality, we can consider the domain to be the Boolean hypercube, $\mcX = \{\pm 1\}^d$. Otherwise, we could map the domain to the hypercube. For an appropriate threshold $t$, we'll define
\begin{align*}
    \mcA(x_1, \ldots, x_n) = \begin{cases}
        1 & \text{if for every $x_i$, there is an $x_j$ with $j \neq i$ s.t. $\langle x_i, x_j \rangle \geq t$,}\\
        0 & \text{otherwise.}
    \end{cases}
\end{align*}
Let $\mcD$ be uniform over $\mcX$. For any $\eps > 0$, $n$, and $d$, we'll show that there is a choice of $t$ such that:
\begin{enumerate}
    \item \Cref{lem:subsample_oblv}: $\Obliviousmax_{\dadd,\eta}(\mathcal{A},\mathcal{D},n) \leq \violet{\eps/2}$, meaning, for any $\wh{\mcD} = (1 - \eta)\mcD + \eta \mcE$, it is the case that $\E_{\bS \sim \wh{\mcD}^n} [\mcA(\bS)] \leq \eps/2$.
    \item \Cref{lem:subsample_adapt}: $\Adaptivemax_{\dadd,\eta}(\mathcal{A},\mathcal{D},m) \geq 1 - \frac{\eps}{2}$ whenever $m =O_\eta(n \log |\mcX| / \log^2 n)$, meaning that for $\bS \sim \mcD^m$, the adaptive adversary can choose $\floor{m \cdot \eta/(1-\eta)}$ points $\bT$ to add to the sample for which
    \begin{align*}
         \E[\mcA_{\mathrm{sub}}(\wh{\bS})] \geq 1 - \lfrac{\eps}{2}\quad \text{where }\wh{\bS} = \bS \cup \bT.
    \end{align*}
\end{enumerate}
Together, these prove that $\mcA$ in the presence of $(\dadd,\eta)$-oblivious adversaries is not $(n, m, 1-\eps)$-equivalent to $\mcA_{\mathrm{sub}}$ in the presence of $(\dadd,\eta)$-adaptive adversaries.

\begin{lemma}\label{lem:subsample_oblv}
For any distribution $\mcE$ and $\wh{\mcD} = (1 - \eta)\mcD + \eta \mcE$,
\begin{align*}
    \Ex_{\bS \sim \wh{\mcD}^n} [\mcA(\bS)] < \eps/2.
\end{align*}
\end{lemma}

\begin{proof}
First, we note that for any $x' \in \mcX$ and a clean sample $\bx \sim \mcD$, the probability that $\langle \bx, x' \rangle \geq t$ is small: by Hoeffding's inequality,  $\Pr_{\bx \sim \mcD}[\langle \bx, x' \rangle \geq t] \leq \exp(-t^2/2d)$. Combining this with a simple union bound, we can show that the probability of even a single clean point forming a correlated pair in the sample is small:
\begin{align*}
    \Ex_{\bS\sim \wh{\mcD}^{n}} [\mcA(\bS)] 
    &\leq \eta^n + \Ex_{\bS\sim \wh{\mcD}^{n}} [\mcA(\bS) \mid \text{at least one clean point in $\bS$}]  \\
    &= \eta^n + \Ex_{\substack{\bx \sim \mcD \\ \bS \sim \wh{\mcD}^{n-1}}} [\mcA(\bS \cup \{\bx\})]  \\
    &\leq \eta^n + \Prx_{\substack{\bx \sim \mcD \\ \bS\sim \wh{\mcD}^{n-1}}} [\text{$\exists \bx' \in \bS$ with $\langle \bx, \bx' \rangle \geq t $}] \tag{weaken to just one clean point}\\
    &\leq \eta^n + (n - 1) \Prx_{\substack{\bx \sim \mcD \\ \bx' \sim \wh{\mcD}}} [\langle \bx, \bx' \rangle \geq t ] \tag{union bound over $S$}\\
    &\leq \eta^n + n \exp \paren{-\frac{t^2}{2d}} \tag{Hoeffding's\violet{, over the randomness of $\bx$}}.
\end{align*}

This will be \violet{vanishingly small} for the particular choice of $t$ determined in \Cref{lem:subsample_adapt}.
\end{proof}

\begin{lemma}\label{lem:subsample_adapt}
For any $m = O_\eta(nd/\log^2 n)$, there exists an adversarial strategy that, given $\bS \in \mcD^m$, chooses $\floor{m \cdot \eta/(1-\eta)}$ points $\bT$ to add to the sample for which
    \begin{align*}
         \E[\mcA_{\mathrm{sub}}(\wh{\bS})] \geq 1 - \eps/2 \quad \text{where }\wh{\bS} = \bS \cup \bT
    \end{align*}
\end{lemma}

\begin{proof}
Let $C = \violet{\floor{m \cdot \eta/(1-\eta)}}$ be the number of points the adversary can add and denote the sample $\bS =  \{ \bx^{(1)}, \ldots \bx^{(m)}\}$. The adversary constructs $\bT = \{\by^{(1)}, \ldots, \by^{(C)}\}$ 
by setting each $\by^{(j)}$ to be the elementwise majority of $k$ chosen points from $\bS$ (where $k$ will be determined later). The idea is that $\by^{(j)}$ is a cluster center that will form a correlated pair with every one of these $k$ points, with high probability.

More formally, for each $j = 1, \ldots, C$, define $\bS^{(j)} = \{x^{(1 + (j-1)k \pmod{|\bS|})}, \ldots, x^{(jk \pmod{|\bS|})}\}$ to be the $j$-th chunk of $k$ points from $\bS$, with the indices wrapping around to the start of $\bS$ as necessary. We take $\by^{(j)}$ to be the elementwise majority\footnote{We can assume $k$ is odd for simplicity.} of the points in $\bS^{(j)}$:

\begin{equation}
    \by^{(j)}_{\ell} \coloneqq \underset{\bx \in S^{(j)}}{\text{Maj}} \{\bx_\ell\}\qquad \text{for $\ell = 1, \ldots, d$.}
\end{equation}

First, we note that for a given $\by^{(j)}$, with high probability, any point in $\bx \in \bS^{(j)}$ will have a large dot product with $\by^{(j)}$: 
\begin{align*}
    \Ex_{\bS^{(j)} \sim \mcD^k}[ \langle \bx, \by^{(j)} \rangle] &=  
    \sum_{\ell = 1}^d \Ex_{\bS^{(j)} \sim \mcD^k}\left[   \bx_{\ell} \by^{(j)}_{\ell} \right] \\
    &= \violet{ \sum_{\ell = 1}^d \paren{ \Prx_{\bS^{(j)} \sim \mcD^k}\left[   \bx_{\ell} = \by^{(j)}_{\ell} \right] - \Prx_{\bS^{(j)} \sim \mcD^k}\left[   \bx_{\ell} \neq \by^{(j)}_{\ell} \right]} }\\ 
    &= \violet{  \sum_{\ell = 1}^d \paren{ \Prx_{\bu \sim \text{Bin}(k-1, \frac{1}{2})}\left[  \bu \geq \frac{(k-1)}{2} \right] - \Prx_{\bu \sim \text{Bin}(k-1, \frac{1}{2})}\left[  \bu < \frac{(k-1)}{2} \right]} }\\
    &=\violet{  d \Prx_{\bu \sim \text{Bin}(k-1, \frac{1}{2})}\left[  \bu = \frac{(k-1)}{2} \right] }\\
    &=\sqrt{\frac{2}{\pi}}{\frac{d}{\sqrt{k}}}(1 \pm o(1)).
\end{align*}
Define $\mu \coloneqq (\sqrt{2/\pi})(d/\sqrt{k})$. We will take $t =  \mu/2$ as the threshold for $\mcA$, both here and in \Cref{lem:subsample_oblv} as well. For $\bS^{(j)} \sim \mcD^k$, with $\by^{(j)}$ the elementwise majority of $\bS^{(j)}$, and any $\bx \in S^{(j)}$, this gives:

\begin{align}
     \nonumber. \Prx_{\bS^{(j)} \sim \mcD^k}\left[ \langle \bx, \by^{(j)} \rangle < t \right] &=
     \Prx_{\bS^{(j)}  \sim \mcD^k}\left[ \langle \bx, \by^{(j)} \rangle < \frac{\mu}{2} \right]  \\
     \nonumber &\leq \exp\left[-\Theta\paren{\frac{\mu^2}{d}}\right] \tag{Hoeffding's inequality} \\
     &= \exp\left[-\Theta\paren{\frac{d}{k}}\right]. \label{eq:cluster_dot}
\end{align}

The subsampling filter $\Phi_{* \to n}$ takes a random subsample of size $n$ from $\wh{\bS} = \bS \cup \bT$ (with replacement). We want to show, with high probability over size $n$ subsamples $\bS' \sim \mcU(\wh{\bS})^n$, that $\mcA(\bS') = 1$. \violet{For any point $\bx \in \bS$, we say $\by^{(j)} \in \bT$ is ``good" for $\bx$ if $\bx$ was in the cluster used to compute $\by^{(j)}$, meaning $\bx \in \bS^{(j)}$. Similarly, for any $\by^{(j)}  \in \bT$, we say that $\bx \in \bS$ is ``good" for $\by^{(j)}$ is $\bx \in \bS^{(j)}$.} By construction, a given $\bx \in \bS$ participates in the computation of at least $\floor{Ck/m} = \Theta_{\violet{\eta}}(k)$ many cluster centers $\by^{(j)}$'s, \violet{and so for each $\bx \in \bS$, there are $\Theta_{\violet{\eta}}(k)$ good $\by\text{'s} \in T$. Similarly, there are exactly $k$ good $\bx\text{'s} \in S$ for each $\by \in T$}. As $\wh{\bS}$ has size $\Theta_{\eta}(m)$, for any $\bx \in \bS'$, using  $\bc_{\bx}$ to denote the number of good points for $\bx$ that show up in $\bS'$, we have that $\bc_{\bx}$ is distributed as $\text{Bin}(n, \Omega_\eta(k/m))$. This gives:
\begin{align*}
    \Prx_{\bS' \sim \mcU(\wh{\bS})^n}[\mcA(\bS') = 0] &= \Prx_{\bS' \sim \mcU(\wh{\bS})^n}[\text{$\exists \bx \in \bS'$ s.t. $\forall \bx' \in \bS', \langle \bx, \bx' \rangle < t$} ] \\ 
    &\leq n \Prx_{\bx, \bS' \sim \mcU(\wh{\bS})^n}[\text{$\forall \bx' \in \bS', \langle \bx, \bx' \rangle < t$} ] \tag{union bound} \\
    &\leq n \left[\Pr[\bc_{\bx} = 0] +  \Prx_{\bx, \bS' \sim \mcU(\wh{\bS})^n}[\text{$\forall \bx' \in \bS', \langle \bx, \bx' \rangle < t$} \mid \bc_{\bx} \geq 1]\right]  \\
    &\leq n \left[\Pr[\bc_{\bx} = 0] + \Prx_{\bx, \bS' \sim \mcU(\wh{\bS})^n}[\text{$\langle \bx, \by_{\bx} \rangle < t$]} \right]  \tag{weaken to good point} \\
    &\leq n \paren{1 - \Theta_\eta\paren{\frac{k}{m}}}^n + n \exp\left[-\Theta_\eta\paren{\frac{d}{k}}\right].   \tag{from \Cref{eq:cluster_dot}}
\end{align*}

\end{proof}

\paragraph{Setting of parameters.} 

To complete the proof of \Cref{thm:subsample}, we need to choose $m$ and $k$ so that, all the terms in \Cref{lem:subsample_oblv} and \Cref{lem:subsample_adapt} are vanishingly small. In particular, we need $n \paren{1 - \Theta_\eta\paren{{k}/{m}}}^n \to 0$ and $ n \exp\paren{-\Theta_\eta\paren{{d}/{k}}} \to 0$ as $n \to \infty$. If $k = \Omega_\eta(m \log n / n)$ and $d = \Omega_\eta(k \log n) = \Omega_\eta(m \log^2 n /n )$ with sufficiently large constant factors, we get the desired result. In other words, as long as $m \le O_\eta(nd/\log^2 n)$,  the adaptive adversary is stronger than the oblivious adversary.


%% file: TechnicalRemarks.tex
\section{Other standard noise models}
\label{sec:other-noise-models}
Here, we list other standard noise models and show they fall under our framework.
\begin{definition}[Nasty classification noise~\cite{BEK02}]
    Let the domain be $\mcX = X \times Y$. Given a size-$n$ sample $S \in \mcX^n$ and a corruption budget $\eta$, the adaptive nasty classification noise adversary is allowed to choose $\floor{\eta n}$ points and for each one, change it from $(x,y)$ to $(x, \wh{y})$ for arbitrary $\wh{y} \in Y$.
\end{definition}

This model is captured by the cost function: 
\begin{align*}
    \mathrm{cost}_{\mathrm{agn}}(\mcD, \wh{\mcD}) = \begin{cases}
        \infty & \text{if $\mcD$ and $\mcD'$ do not have the same marginal distribution over $\mcX$} \\
        \TV(\mcD, \wh{\mcD}) & \text{otherwise}
    \end{cases}
\end{align*}

The $(\mathrm{cost}_{\mathrm{agn}}, \eta)$-oblivious adversary corresponds exactly to the well-studied agnostic learning model~\cite{Hau92,KSS94}. Hence, \Cref{thm:SQ} implies that nasty classification noise and the agnostic learning model are identical for SQ algorithms.

The final noise model that we discuss is defined with respect to an adversary that has intermediate adaptive power:
\begin{definition}[Malicious noise~\cite{Val85}]
    \label{def:mal-noise}
    In the malicious noise model where the adversary has corruption budget $\eta$, a sample is generated point-by-point. For each point, independently with probability $1 - \eta$, that point is $\bx \sim \mcD$. Otherwise, the adversary is allowed to make that point an arbitrary $x \in \mcX$ with knowledge of the previous points sampled but not the future points.
\end{definition}

\paragraph{On the relationship between malicious noise and additive noise.} The malicious noise adversary does not have full adaptivity, as when they decide what point to add, they only have knowledge of previous points sampled and not future ones. We now show how to encode the fully adaptive version of malicious noise, in which the adversary knows all points in the sample when deciding corruptions, in our  framework.

We first augment the domain to $\mcX' = \mcX \cup \{\varnothing\}$ where $\varnothing$ will be used to indicate the adversary can change this point arbitrarily. We then let $\mcD'$ be the distribution satisfying, for each $x \in \mcX$:
\begin{align*}
    \mcD' = (1 - \eta)\mcD + \eta \mcD_{\varnothing}
\end{align*}
where $\mcD_{\varnothing}$ is the distribution that always outputs $\varnothing$. We define the cost function to be
\begin{align*}
    \cost_{\mathrm{mal}}(\mcD', \wh{\mcD}') = \begin{cases}
    \infty & \text{if $\Prx_{\bx \sim \wh{\mcD}'}[\bx = \varnothing] > 0$} \\
    \infty & \text{if $\Prx_{\bx \sim \wh{\mcD}'}[\bx = x] < \Prx_{\bx \sim \wh{\mcD}}[\bx = x]$ for any $x \in \mcX$} \\
    0 &\text{otherwise.}
    \end{cases}
\end{align*}
Note that in order for $\wh{\mcD}'$ to be a valid corruption of $\mcD'$ (i.e. $\cost_{\mathrm{mal}}(\mcD', \wh{\mcD}') \neq \infty$), all of the probability mass of $\wh{\mcD}'$ must be over $\mcX$, with none on $\varnothing$. The above provides an encoding of an adaptive adversary that is at least as powerful as malicious noise. To further understand the corresponding oblivious adversary, we note that, after fixing $\eta$,
\begin{align*}
     \cost_{\mathrm{mal}}(\mcD', \wh{\mcD}') \neq \infty \quad\text{if and only if}\quad \wh{\mcD}' = (1-\eta)\mcD + \eta \mcE \quad \text{for some distribution $\mcE$}.
\end{align*}
Hence, the oblivious adversary corresponding to malicious noise is the same as the oblivious adversary for additive noise. \Cref{thm:add} implies that the adaptive and oblivious versions of additive noise, as well as malicious noise, are all equivalent.

\section{Technical remarks}
\label{sec:technical remarks} 


\paragraph{Fixed budget vs. \violet{variable} budget.}  Consider the nasty noise model (\Cref{def:strong-cont}), corresponding to $\rho = \TV$ in our framework. Given a size-$n$ clean sample $\bS$, the adaptive adversary can choose any $\eta$-fraction of the points to change arbitrarily to create the corrupted sample $\wh{\bS}$. Often an alternative definition is used where the adversary is allowed to arbitrarily change  $m$ points in $\bS$, where $m$ can vary based on the specific sample $\bS$, as long as the marginal distribution of $m$ over samples $\bS \sim \mcD^n$ is $\mathrm{Bin}(n, \eta)$. This definition is used by \cite{DKKLMS19,ZJS19} to show that the adaptive adversary can simulate any oblivious adversary. Technically, for our definition of an adaptive adversary with fixed budget $\eta$, this fact is not strictly true.\footnote{For example, suppose $\mcX = \zo$, $\mcD$ is the identically $0$ distribution, and $\eta = 0.1$. If a size-$10$ sample is taken, under our adaptive definition, $\wh{S}$ will never have more than a single $1$. However, the oblivious adversary can choose $\wh{\mcD}$ that returns $1$ with probability $0.1$ and as a result is a non-zero chance that two $1$'s appear in a size-$10$ sample.} However, all our results are readily extendable to these slightly stronger adaptive adversaries with random-budgets.

We define an adversary model which encompasses those adversaries with variable budgets. First, the adversary will generate a corrupted data set $\wh{\bS}$. Then, it is allowed to change some points in that set to create $\wh{\bS}'$ as long as, most of the time,  $\wh{\bS}'$ and $\wh{\bS}$ are close.

\begin{definition}[Strong $(\rho, \eta)$-adaptive adversary]
    \label{def:strong-adaptive}
    Given a cost function $\rho$ and budget $\eta$, an $n$-sample algorithm operating in the \emph{strong} $(\rho ,\eta)$-adaptive adversary model will receive as input the sample $\wh{\bS}'$ where
    \begin{enumerate}
        \item If the true data distribution is $\mcD$, first a clean sample $\bS \sim \mcD^n$ is generated.
        \item The adversary chooses a $\wh{\bS}$ satisfying $\rho(\mcU(\bS), \mcU(\wh{\bS})) \leq \eta$.
        \item The adversary chooses some $\wh{\bS}'$ where, over the randomness of the original sample and the adversaries decisions, the following holds.
        \begin{align}
            \label{eq:strong-adaptive}
        \Prx_{\wh{\bS}, \wh{\bS}'}\left[\TV(\mcU(\wh{\bS}), \mcU(\wh{\bS}')) \geq t\right] &\leq \exp(-O(nt^2)) \quad\quad\text{for all $t \in (0,1)$}.
    \end{align}
    \end{enumerate}
\end{definition}

All of our upper bounds on the strength of adaptive adversaries also apply to the strong adaptive adversary. See \Cref{remark:SQ-strong} for changes in the proof needed for \Cref{thm:SQ}. For \Cref{thm:add,thm:if-can-be-done}, we can use an even weaker restriction on the adversary and only require that
\begin{align}
    \label{eq:strong-adaptive-sqrt}
    \Ex_{\wh{\bS}, \wh{\bS}'}\left[\TV(\mcU(\wh{\bS}), \mcU(\wh{\bS}')) \right] &= O(1/\sqrt{n})
\end{align}
in place of \Cref{eq:strong-adaptive}. See \Cref{remark:strong-adaptive-subsampling} for how \Cref{eq:strong-adaptive-sqrt} can be used to make \Cref{thm:add,thm:if-can-be-done} work with strong adaptive adversaries.

\paragraph{Finite vs infinite domains.}
For our analyses, we assume that the domain, $\mcX$, is finite. Our goal is to understand when an algorithm that succeeds in the presence of an oblivious adversary implies an algorithm that succeeds in the presence of an adaptive adversary. Any algorithm that succeeds in the presence of an oblivious adversary can only read finitely many bits of each data point, effectively discretizing the domain.

That said, \Cref{thm:SQ} also applies to infinite domains. If the domain is infinite, a more general definition of closed under mixtures is required in place of the simpler \Cref{def:closed-under-mixtures}
\begin{definition}[Closed under mixtures, infinite domain]
    \label{def:closed-under-mixtures-infinite}
    We say that $\rho$ is \emph{closed under mixtures} if for any distributions $\mcD, \mcD'$, and coupling of $\bx \sim \mcD$, $\bx' \sim \mcD'$ and a latent variable $\bz$ (over any domain),
\[ 
    \rho(\mcD, \mcD') \leq \sup_{\bz} \left\{ \rho((\bx \mid \bz), (\bx' \mid \bz))\right\}.
   \] 
\end{definition}
When the domain is finite, \Cref{def:closed-under-mixtures,def:closed-under-mixtures-infinite} are equivalent. When it is infinite, \Cref{def:closed-under-mixtures-infinite} is needed to prove \Cref{fact:sample-to-dist}. The remainder of the proof of \Cref{thm:SQ} is identical.

%% file: appendix.tex
\section{Proof of \Cref{fact:sample-to-dist}}

\label{appendix:sample-to-dist}
\begin{proof}
    By a standard inductive argument, \Cref{def:closed-under-mixtures} implies that for any $m \in \N$, weights $\theta_1, \ldots, \theta_m \geq 0$ summing to $1$ and distributions $\mcD_1, \ldots, \mcD_m, \wh{\mcD}_1, \ldots, \wh{\mcD}_m$, that
    \begin{align*}
        \rho \left(\sum_{i \in [m]}\theta_i \mcD_i,  \sum_{i \in [m]}\theta_i \wh{\mcD}_i\right) \leq \max_{i \in [m]}\left\{\rho\left(\mcD_i, \hat{\mcD}_i\right)\right\}.
    \end{align*}
    The distribution $\mcD$ and $\hat{\mcD}$ can be written as the mixtures
    \begin{align*}
        \mcD &= \sum_{S \in \mcX^n}\Prx_{\bS \sim \mcD^n}[\bS = S] \mcU(S),\\
        \wh{\mcD} &= \sum_{S \in \mcX^n}\Prx_{\bS \sim \mcD^n}[\bS = S] \mcU(\wh{S}).
    \end{align*}
    Since $\rho(\mcU(S), \mcU(\hat{S})) \leq \eta$ for every $S \in \mcX^n$, we can conclude $\rho(\mcD, \wh{\mcD}) \leq \eta$.
\end{proof}